\documentclass{article} 
\usepackage{arxiv_conference,times}

\usepackage{amsmath,amsfonts,bm}

\def\eqref#1{equation~\ref{#1}}

\def\1{\bm{1}}

\def\vzero{{\bm{0}}}

\def\vtheta{{\bm{\theta}}}

\def\ve{{\bm{e}}}

\def\vg{{\bm{g}}}

\def\vm{{\bm{m}}}

\def\vu{{\bm{u}}}
\def\vv{{\bm{v}}}

\def\mH{{\bm{H}}}
\def\mI{{\bm{I}}}

\def\mP{{\bm{P}}}

\DeclareMathAlphabet{\mathsfit}{\encodingdefault}{\sfdefault}{m}{sl}
\SetMathAlphabet{\mathsfit}{bold}{\encodingdefault}{\sfdefault}{bx}{n}

\def\gL{{\mathcal{L}}}

\newcommand{\R}{\mathbb{R}}

\usepackage{graphicx}
\usepackage[T1]{fontenc}
\usepackage[utf8]{inputenc}
\usepackage{booktabs}
\usepackage{amsthm}
\usepackage{subcaption}
\usepackage{wrapfig}

\usepackage{hyperref}
\usepackage{url}
\usepackage{xcolor}
\definecolor{mydeepblue}{RGB}{26,64,149}
\definecolor{mygreen}{RGB}{37,145,60}
\definecolor{myorange}{RGB}{217,113,26}
\definecolor{mypurple}{RGB}{117,60,204}
\hypersetup{
    colorlinks=true,
    linkcolor=myorange,
    citecolor=mygreen,
    urlcolor=myorange,
    filecolor=myorange
}

\title{The Road Taken: The Role of Optimizers\\at the Edge of Stability}

\author{Jaerin Lee \& Kyoung Mu Lee\\
Computer Vision Lab, ASRI\\
Seoul National University\\
Seoul 08826, Korea \\
\texttt{\{ironjr,kyoungmu\}@snu.ac.kr}
}

\newtheorem{theorem}{Theorem}[section]
\newtheorem{lemma}[theorem]{Lemma}
\newtheorem{proposition}[theorem]{Proposition}
\newtheorem{corollary}[theorem]{Corollary}
\newtheorem{remark}[theorem]{Remark}

\iclrfinalcopy 
\begin{document}

\maketitle

\begin{abstract}
The edge of stability refers to a phenomenon in deep learning with gradient-based optimizers where the Hessian eigenvalues of the loss remain stable above a threshold that the classical descent lemma predicts to be unstable.
Previous works formulate the edge of stability with respect to the maximum Hessian eigenvalue and the learning rate.
However, we observe that many first-order methods, including gradient descent, significantly violate the stability bound predicted by these theories by a factor as large as $\times 21.1$.
Moreover, this deviation turns out to be systematic and highly dependent on the underlying optimizer, which is not captured by previous formulations.
This calls for a new formulation of the stability threshold, which we derive from the directional Hessian and the gradient-alignment score with respect to the actual update taken by the optimizer, rather than the maximum curvature mode.
Our new formulation of the \emph{realized} edge of stability not only removes optimizer-dependent offsets and provides more consistent predictions of the stability threshold, but also introduces new diagnostic tools that reveal the unique role of the optimizer in actively balancing between the temporal and spatial budgets in first-order optimization.
\end{abstract}

\section{Introduction}
\label{sec:introduction}
The training dynamics of deep learning optimization still possess many mysteries.
One of the most intriguing phenomena is the \emph{edge of stability} (EoS)~\citep{cohen2021gradient}, where gradient descent operates in a region where classical theory predicts instability.
The good old descent lemma~\citep{ortega1970iterative,bertsekas1999nonlinear} suggests that the available step size $\eta$ for gradient descent is bounded by the curvature of the objective function:
\begin{equation}
\label{eq:descent_lemma}
\eta \lambda_{\max} (\mH) \leq 2,
\end{equation}
where $\lambda_{\max}$ is the maximum eigenvalue of the Hessian matrix $\mH = \nabla_\vtheta^2 \gL(\vtheta^\star)$ of the loss $\gL(\vtheta)$ at a local minimum $\vtheta^\star$.
Although this bound holds for convex objectives, its behavior becomes more subtle in highly nonconvex settings such as deep learning.
In those regimes, it has been reported that the relation (\ref{eq:descent_lemma}) behaves more like a stable attractor that draws the curvature $\lambda_{\max}$ towards a value near $2 / \eta$, rather than a strict stability threshold~\citep{lewkowycz2020large,cohen2021gradient}.
During the first few hundred iterations, the top Hessian eigenvalue $\lambda_{\max}$ either shrinks rapidly~\citep{lewkowycz2020large,zhu2024catapults} or increases gradually~\citep{cohen2021gradient} depending on the initialization, until it eventually reaches a stable point near $2 / \eta$, and then hovers around this value.
This mechanism is referred to as the \emph{edge of stability} (EoS)~\citep{cohen2021gradient}.

Rather than being unified into a single definition, previous works have suggested various alternative formulations for the edge of stability~\citep{lyu2022understanding,cohen2023adaptive,damian2023self,agarwala2023second,zhu2023understanding,andreyev2024edge,chen2024stability,cohen2025central,islamov2026noneuclidean,andreyev2026momentum,litman2026origin}.
Nevertheless, all share the same core idea and structure as a product of two or three separable quantities: (1) a learning rate scalar $\eta$, (2) a \emph{maximum curvature eigenmode} or \emph{sharpness} $S$ with respect to a (possibly preconditioned) Hessian, and optionally (3) a momentum-induced gain $\Gamma$ if the optimizer is stateful:
\begin{equation}
\label{eq:edge_of_stability}
\eta S \Gamma \leq 2.
\end{equation}
For instance, the original formulation \citep{cohen2021gradient} assigns $S = \lambda_{\max}(\mH)$ and $\Gamma = 1$.
Later, \citet{cohen2023adaptive} incorporate preconditioning by taking $S = \lambda_{\max}(\mP^{-1/2} \mH \mP^{-1/2})$ and introduce a factor $\Gamma = \Gamma(\beta)$ to account for the momentum, e.g., $\Gamma(\beta) = 1 / (1 + \beta)$ for Heavy Ball~\citep{polyak1964some} and $\Gamma(\beta) = (1 + 2 \beta) / (1 + \beta)$ for Nesterov's momentum~\citep{nesterov1983method}.

In summary, the edge of stability has been understood as a product of the \emph{worst-mode} sharpness and \emph{separable} temporal factors.
This formulation provides \emph{available} stability thresholds that are provably marginal for quadratic objectives.
However, we find that the actual relative position of the edge, $\chi := \eta S \Gamma / 2$, is not consistent with the theoretical prediction $\chi = 1$.
As Figure~\ref{fig:observation} shows, its values range from 0.79 to 21.07 across different types of optimizers and learning rates.
Even vanilla GD sits above $\chi = 1$ (top left).
The median of the saturated curve grow up to $\chi \approx 1.62$ as the learning rate increases.
More complicated optimizers, such as dual-momentum optimizers~\citep{lee2024grokfast,pagliardini2025ademamix} or pole-zero filters~\citep{ma2018quasi}, show their actual EoS rising to around $\chi = 2.65$ (columns 3 and 4).
The most extreme case is AdaGrad~\citep{duchi2011adaptive} (top right), which shows a clear edge starting from $\chi = 1.42$ at $\eta = 0.01$ and reaching $\chi = 21.07$ at $\eta = 0.1$.
The deviation of $\chi$ from the theoretical prediction of $\chi = 1.0$ depends strongly on the underlying optimization algorithm.
This systematic deviation poses a new challenge to the current formulation and interpretation of the edge of stability.

In fact, this formulation (\ref{eq:edge_of_stability}) is hiding a 
critical assumption: the worst-mode sharpness $S$ is always realized by 
the optimizer, which is not always the case.
Starting from modeling the actual dynamical system the optimizer operates 
on, we derive a new formulation of the stability threshold from the 
directional Hessian~\citep{lee2023ias,mishkin2024directional,
islamov2026noneuclidean} and the gradient-update alignment~\citep
{lee2025greedy}.
As shown in the bottom row of Figure~\ref{fig:observation}, this new formulation, called the \emph{realized} edge of stability $\zeta$, largely eliminates optimizer-dependent offsets.
Furthermore, it is much more efficient to compute and recovers the original formulation $\chi$ by decomposing it into interpretable factors.
This provides a richer understanding of the phenomenon itself and reveals the optimizer's role in steering along the edge of stability by actively calibrating the temporal gauge and spatial budget.

\begin{figure}[t]
\centering
\includegraphics[width=\linewidth]{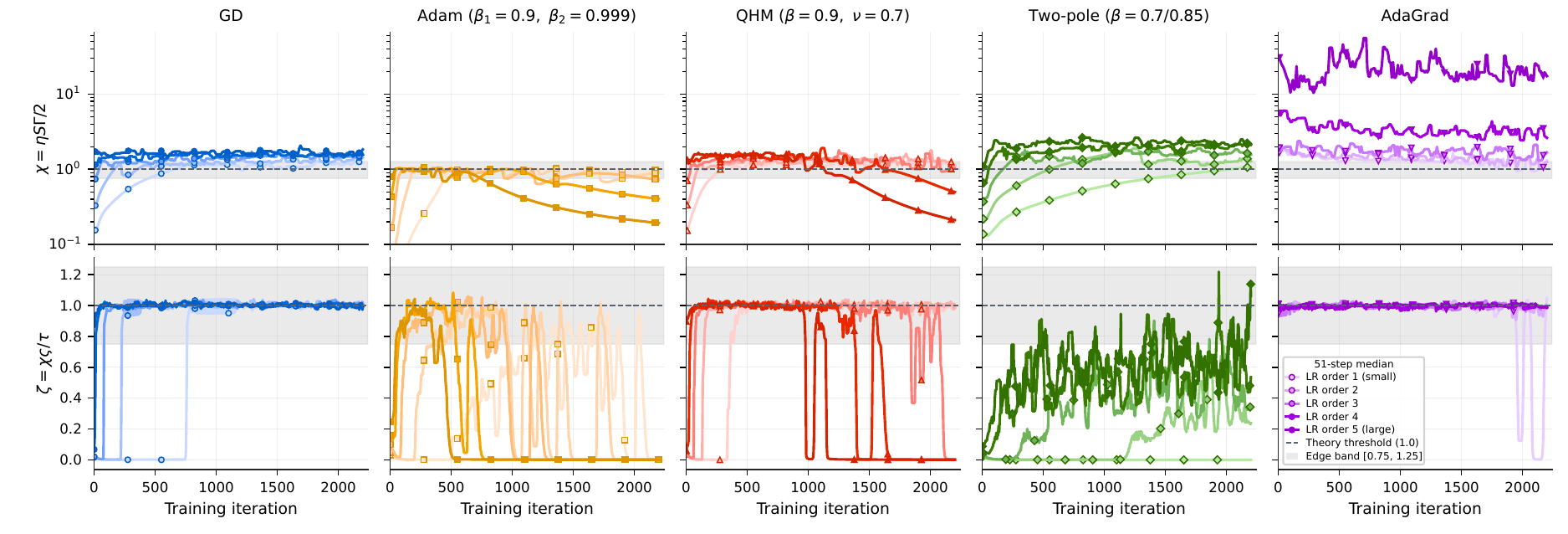}
\vspace{-1.8em}
\caption{%
\small
\textbf{Edge of stability across different families of optimizers.}
We train a fully-connected network on CIFAR-10 using various full-batch gradient-based optimizers: gradient descent (GD), Adam, GD with a pole-zero filter (QHM) and GD with a two-pole filter (Grokfast), and AdaGrad.
Theoretical prediction of the edge of stability $\chi := \eta S \Gamma / 2 = 1.0$ is occasionally violated in practice, and the offsets $\chi$ are heavily dependent on the type of optimizers in use.
This systematic deviations motivate a new formulation of the \emph{realized} edge of stability $\zeta := \chi \varsigma / \tau$ with $\varsigma$ and $\tau$ being the spatial and temporal calibrations, respectively.
This new formulation is more consistent with the empirical observations and removes the optimization-dependent offsets.
}
\label{fig:observation}
\end{figure}

Our contributions are:
\begin{itemize}
    \item We extend the original formulation of the edge of stability~\citep{cohen2021gradient,cohen2023adaptive} to general gradient-based optimizers with states, which is exact for fixed quadratic objectives.
    \item We derive a new formulation of the stability threshold, based on the directional Hessian and the gradient-update alignment score~\citep{lee2025greedy}, which is more consistent with empirical observations and eliminates optimizer-dependent offsets.
    \item We factorize the realized edge of stability $\zeta$ into interpretable components: the original edge of stability $\chi$, spatial participation $\varsigma$, and temporal calibration $\tau$, and show that the two EoS formulations coincide, $\zeta = \chi$, if and only if $\varsigma = \tau$.
    \item This provides a simple yet powerful diagnostic tool for understanding the main factors driving a training trajectory at the edge of stability, revealing how much of the potential loss decrease, $\vg^\top \vu = -(1/\eta) \mathrm d\gL / \mathrm dt$, carried out by the optimizer is actually converted into real progress $-\Delta \gL/\eta$, and how much is lost to different components.
\end{itemize}

\section{Preliminaries: Edge of Stability}
\label{sec:preliminary}
To understand the edge of stability, we start by reviewing its previously established formulations.
Following the aforementioned lines of work, we focus on the full-batch settings where the analysis is more tractable and the effect is more pronounced.
Let us define a few notations to clarify the following discussion.
In gradient-based optimization, we train a parameter $\vtheta$ by minimizing a loss function $\gL(\vtheta)$ through its first-order gradients $\vg := \nabla_{\vtheta} \gL(\vtheta)$.
We can either use this gradient directly, or filter it by some input-dependent preconditioner $\mP$ and temporal filter $Q$ to produce a parameter update $\vu$ which has the same shape as the parameter $\vtheta$:
\begin{equation}\tag{{\color{blue!50}$\spadesuit$}}
\label{eq:gradient_filter}
\renewcommand{\boxed}[1]{\colorbox{lightgray!15}{\ensuremath{#1}}}
\boxed{\ %
\vtheta_{t+1} \;=\; \vtheta_t - \eta \vu_t \;, \quad \text{where} \quad \vu_t \;=\; \mP_t^{-1} (Q_t * \vg_{\le t})_t \;.
}
\end{equation}
With an input-dependent preconditioner $\mP_t$ and a causal gradient filter $Q_t$ acting on a history of gradient signal $\vg_{\le t}$, this gives a generic formulation for gradient-based optimizers with internal states.
This not only includes GD, Heavy Ball~\citep{polyak1964some}, Nesterov's Momentum~\citep{nesterov1983method}, and various preconditioned versions of these~\citep{duchi2011adaptive,tieleman2012lecture,kingma2015adam,loshchilov2019decoupled}, but also extends to more exotic temporal filters such as dual-momentum~\citep{lee2024grokfast,pagliardini2025ademamix} or quasi-hyperbolic momentum (QHM)~\citep{ma2018quasi}.
For example, GD corresponds to $\mP_t = I$ and $Q_t = I$, and Heavy Ball with momentum coefficient $\beta$ is repesented by $Q * \vg_{\le t} = \sum_{k=0}^{t} \beta^k \vg_{t-k}$.
Signal processing notation for the gradient filter $Q$ allows us to study this dynamics in frequency domain, i.e., $Q(z) = \sum_{t=0}^{\infty} Q_t z^{-t}$.
We also assume its dc gain is normalized to $Q(1) = 1$, so that the learning rate $\eta$ is solely responsible for the gain of the system.
This simplifies the analysis without loss of generality.

From extensive observations, \citet{cohen2021gradient,cohen2023adaptive} formulated the EoS as a \emph{separable} product of a function of the preconditioner $\mP$ and a function of the momentum coefficient $\beta$,
\begin{equation}
\label{eq:edge_preconditioner}
\eta \lambda_{\max}({\mP_t^{-1/2} \mH_t \mP_t^{-1/2}})\, \Gamma(\beta) \;<\; 2,
\end{equation}
where $\mH_t$ is the Hessian at time $t$.
For example, Heavy Ball gives $\Gamma(\beta) = 1/(1 + \beta)$, normalized Heavy Ball gives $\Gamma(\beta) = (1 - \beta) / (1 + \beta)$, Nesterov's momentum gives $\Gamma(\beta) = (1 + 2\beta) / (1 + \beta)$, and so on.
It is proven for each momentum type that this behaves as a marginal stability condition for the frozen quadratic system~\citep{cohen2023adaptive}.
The following proposition generalizes their results to any gradient-based optimizer of type (\ref{eq:gradient_filter}).

\begin{proposition}[Edge of Stability for Stateful Optimizers]
\label{prop:edge_preconditioner}
\textnormal{[\texttt{\hyperref[appx:proof:eos_prop]{proof}}]}
Consider a dynamical system given by the update rule (\ref{eq:gradient_filter}).
Let $(\lambda, \vv)$ be any eigenpair of the preconditioned Hessian $\mP_t^{-1/2} \mH_t \mP_t^{-1/2}$.
Then the gain $k_\star(Q)$ at the first crossing of the unit circle by the root locus of $Q(z)$ is given by
\begin{equation}
\label{eq:gain_at_first_crossing}
k_\star(Q) = \min_{\omega \in \Omega_{+}(Q)} \left| \frac{1 - e^{i\omega}}{Q(e^{i\omega})} \right|, \quad \text{where} \quad
\Omega_{+}(Q) \;:=\; \left\{\, \omega \in (0, \pi] : \frac{1 - e^{i\omega}}{Q(e^{i\omega})} \in \mathbb{R}_{>0} \,\right\},
\end{equation}
where $\Omega_{+}(Q)$ is the set of frequencies on the unit circle at which the locus can cross with a real positive gain.
If we write $\Gamma := 2/k_\star(Q)$, then the equivalent stability condition for the frozen quadratic system is attained by
\begin{equation}\tag{{\color{red!50}$\clubsuit$}}
\label{eq:marginal_stability_condition}
\renewcommand{\boxed}[1]{\colorbox{lightgray!15}{\ensuremath{#1}}}
\boxed{\ %
\eta \lambda \Gamma \;\le\; 2.
}
\end{equation}
Moreover, if the first unit-circle crossing is at the Nyquist frequency $z = -1$ with $Q(-1) > 0$, i.e., period-2 oscillation, then $\Gamma = Q(-1)$ and the stability condition becomes $\eta \lambda Q(-1) \le 2$.
\end{proposition}

\begin{corollary}[Recovering \citet{cohen2023adaptive}'s formulae]
\label{cor:cohen_separable_momentum_coefficient}
\textnormal{[\texttt{\hyperref[appx:proof:cohen_gamma]{proof}}]}
The coefficient $\Gamma(\beta)$ in the inequality (\ref{eq:edge_preconditioner}) for momentum-based optimizers is attained at the Nyquist frequency $z = -1$, i.e., $\Gamma(\beta) = Q(-1)$ for (dc gain-normalized) Heavy Ball and Nesterov's momentum.
\end{corollary}

The detailed derivation is provided in Appendix~\ref{appx:proof:eos_prop}.
The gauge $\Gamma$ of each optimizer filter, including the cases $\Gamma \neq Q(-1)$, is derived in Appendix~\ref{appx:proof:gamma_table}.
The sketch of the proof is as follows:
Let the parameter along the $i$-th eigenvector be $x_i = \vv_i \cdot \vtheta$, then the dynamics of the parameter along this direction is given by (omit $i$ for brevity)
\begin{equation}
\label{eq:eigencomponent_dynamics}
x_{t+1} \;=\; x_t - \eta \lambda (Q_t * x_{\le t})_t.
\end{equation}

\begin{wrapfigure}{r}{0.45\linewidth}
\centering
\vspace{-1.7em}
\includegraphics[width=\linewidth]{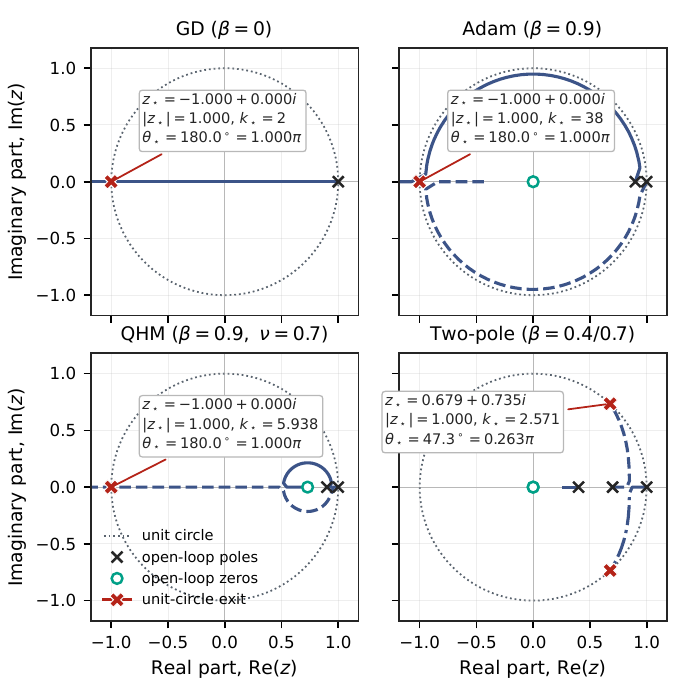}
\caption{\small
\textbf{Root locus of optimizers.}
It is possible to exit the unit circle at $z \neq -1$, i.e., $\Gamma \neq Q(-1)$.
}
\label{fig:rootlocus}
\vspace{.7em}
\end{wrapfigure}

In $z$-domain, we have $X(z) = \sum_{t=0}^{\infty} x_t z^{-t}$ and the convolution becomes a multiplication:
\begin{equation}
\label{eq:eigencomponent_dynamics_frequency}
z X(z) \;=\; X(z) - \eta \lambda Q(z) X(z).
\end{equation}
Dividing both sides by $X(z)$, we obtain the \emph{characteristic transfer function} along this mode:
\begin{equation}
\label{eq:eigencomponent_transfer_function}
z - 1 + \eta \lambda Q(z) \;=\; 0.
\end{equation}
Classical control theory tells us that any linear discrete-time dynamic system must have all its roots $z$ inside the unit circle $|z| < 1$ in order to be stable~\citep{ogata1995discrete}, which gives the results.

From Proposition~\ref{prop:edge_preconditioner} and Corollary~\ref{cor:cohen_separable_momentum_coefficient}, we see that the original EoS~\citep{cohen2021gradient} and adaptive EoS~\citep{cohen2023adaptive} are special cases of this general stability condition.
The boxed bound~(\ref{eq:marginal_stability_condition}) is an \emph{available} stability limit on a frozen quadratic, which is exact in this case, as demonstrated in Appendix~\ref{appx:proof:quadratic_toy}: each eigenmode is independent, $\Gamma$ is a functional of $Q$ alone, and $\chi := \eta S \Gamma / 2 \le 1$ is necessary for no mode to be linearly unstable.
Using $\chi$ as the \emph{observed} hovering edge of a nonlinear trajectory is a different statement, which requires an additional \emph{occupation hypothesis}: over the relevant horizon the update $\vu$ keeps realizing the worst-case mode of that stability limit.
In other words, it stays in $\operatorname{span}(\vv_{\max})$ of a fixed (preconditioned) Hessian $\mP_t^{-1/2} \mH_t \mP_t^{-1/2}$ over the relevant horizon.
Still, the marginal stability condition (\ref{eq:marginal_stability_condition}) remains an exact stability limit for the frozen quadratic system on which the original theory is based.
It does not, by itself, imply that a nonlinear deep learning trajectory must hover at $\chi = 1$.
It is noteworthy that $Q$ does not play a role in the occupation hypothesis.
Section~\ref{sec:observation} tests the occupation hypothesis first on vanilla GD, where $\mP = I$, $Q \equiv 1$ and $\Gamma = 1$, so systematic mispredictions $\chi > 1$ cannot be blamed on a missing momentum or on a preconditioning artefact.

\section{Optimizer-Dependent Offsets}
\label{sec:observation}
The extended formulation of the edge of stability (\ref{eq:marginal_stability_condition}) allows us to test Cohen's formulation of EoS in various optimizers.
We train a 5-layer MLP with 200 hidden units with $\mathrm{tanh}$ activation on CIFAR-10~\citep{krizhevsky2009learning} using different full-batch optimizers and learning rates.
Specifically, we use gradient descent (GD), GD with Heavy Ball (GDM), GD with Nesterov's momentum (GDN), Adam~\citep{kingma2015adam}, AdamW~\citep{loshchilov2019decoupled}, AdaGrad~\citep{duchi2011adaptive}, RMSprop~\citep{tieleman2012lecture}, AdaFactor~\citep{shazeer2018adafactor}, NAdam~\citep{dozat2016incorporating}, PAdam~\citep{chen2020closing}, AMSGrad~\citep{reddi2018on}, GD with a quasi-hyperbolic momentum (QHM)~\citep{ma2018quasi}, and GD with a Grokfast-like two-pole filter~\citep{lee2024grokfast}.
For each optimizer, we vary the learning rates on a geometrically scaled grid of five values, depending on the optimizer family.
The results are summarized in Figure~\ref{fig:observation} and Figures~\ref{fig:eos-1-3} and~\ref{fig:eos-4-6}, and Table~\ref{tab:offsets} in Appendix~\ref{sec:experiment}.
To measure the discrepancy between the realized and theoretical edge of stability, we define the \emph{relative} edge of stability $\chi := \eta S \Gamma / 2$ as a unit-free measure.
$\chi = 1$ corresponds to theoretical exactness, and is displayed as dashed lines.
We additionally summarize the range of $\chi$ for different (optimizer, lr) in Figure~\ref{fig:observation_general}.
The range is computed as the median of $\chi$ in the interval between first entering $\chi > 0.75$ and last exiting $\chi < 0.75$ during each training trajectory.

\begin{figure}[t]
\centering
\includegraphics[width=\linewidth]{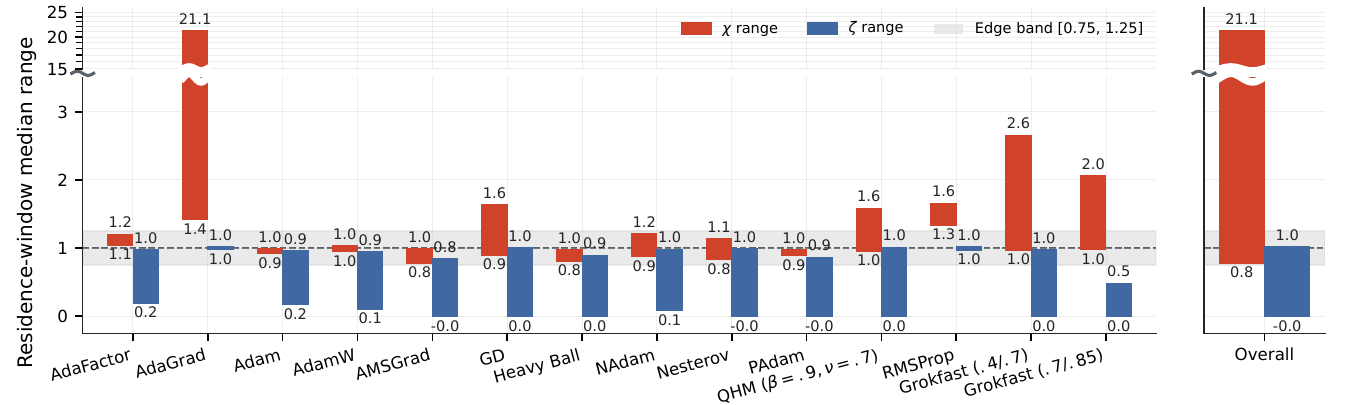}
\vspace{-1.8em}
\caption{\small
\textbf{Family-wise ranges of per-setting interval medians.}
Medians of $\chi$ taken from the interval between first $\chi>0.75$ and the final downward exit.
The previous formulation of the edge of stability $\chi \le 1$~\citep{cohen2021gradient,cohen2023adaptive} is occasionally violated by large, optimizer-dependent offsets.
Our \emph{realized} EoS $\zeta$ has its median universally upper bounded by near 1.0, which is more consistent with the name ``edge'' of stability.
}
\label{fig:observation_general}
\end{figure}

We observe that the deviation of $\chi$ varies in a very large range $\chi \in [0.79, 21.07]$ in a systematic manner.
Different optimizer families exhibit different patterns of deviation, and the offset generally increases with $\eta$, as observed in Figure~\ref{fig:observation} and Figures~\ref{fig:eos-1-3} and~\ref{fig:eos-4-6} in Appendix~\ref{sec:experiment}.
Offsets lower than 1.0 are easy to interpret: either the trajectory remains strictly within the stable region, or the optimizer is not able to drive the trajectory to the edge of stability.
Those with offsets greater than 1.0 are more interesting.
Complex filters, such as QHM and two-pole, raise $\chi$ up to $1.576$ and $2.646$, respectively.
At one extreme, AdaGrad severely violates the $\chi \le 1$ bound to reach $\chi = 21.070$.

Another observation that draws our attention is that even the vanilla full-batch GD violates our expectation to stay near $\chi = 1$: Table~\ref{tab:offsets} reports in-window median $\chi$ of $0.904$, $1.130$, $1.246$, $1.478$, $1.624$ across the five learning rates.
Its four of five cases sit above $\chi = 1$ bound, and the observed edge systematically increases with $\eta$.
Since GD uses $\mP = I$ and $Q \equiv 1$, hence $\Gamma = 1$ and $\chi = \eta S / 2$, this offset cannot be a missing $\Gamma(\beta)$.
By contrast, Heavy Ball, Adam, and AdamW stay inside the $\chi \le 1$ bound, with global median $0.949$, $0.960$, and $0.971$, respectively.
The momentum-induced gauge $\Gamma$ is what those stateful methods relieves from the violation of the edge of stability by GD.

The results not only imply a stabilization mechanism~\citep{damian2023self} that keeps the trajectory stable even outside the guaranteed region of stability by consistently reducing the sharpness, but also suggest that \emph{the attractor of $\chi$ itself lies outside the theoretically predicted stable margin}.
These optimizer-dependent offsets raise a significant concern that the current formulation of the edge of stability is incomplete, especially in how it identifies the empirical edge with $\lambda_{\max}$ rather than with the curvature along the update that is actually taken.

\section{Realized Edge of Stability}
\label{sec:theory}
The previous observation calls for an update to our formulation of the edge of stability to account for the \emph{actual} training dynamics that are representative of the true stability threshold.
Previous works~\citep{lee2023ias,mishkin2024directional,islamov2026noneuclidean} give us a hint to this problem: the worst-case eigenmode $\operatorname{span}(\vv_{\max})$ that defined the boundary of the region of stability in (\ref{eq:marginal_stability_condition}) may not actually be taken as update directions.
This idea encourages us to examine the actual update direction $\vu$ that the optimizer executes during training.
We go back to the generalized dynamical system formulation in equation~(\ref{eq:gradient_filter}), and ask:
\begin{center}
\setlength{\fboxsep}{8pt}
\colorbox{lightgray!15}{\parbox{0.92\linewidth}{
Given the update direction $\vu$ that the optimizer $(\mP, Q)$ actually takes, where is the local quadratic stability boundary along this particular ray?
}}
\end{center}
To answer this question, fix the iteration index $t$ and consider a ray in parameter space from the current parameter $\vtheta$ along the update direction $\vu$, and define the \emph{secant Hessian} along the ray as: 
\begin{equation}
\label{eq:ray_parameterization}
\vtheta(s) \;:=\; \vtheta - s\vu, \qquad s \ge 0, \qquad
\bar{\mH}_{s,\vu} \;:=\; 2 \int_0^1 (1 - r)\, \nabla^2_{\vtheta} \gL(\vtheta - r s \vu) dr.
\end{equation}
Letting $\phi(s) := \gL(\vtheta(s)) = \gL(\vtheta - s\vu)$, we get its Taylor expansion at $s = \eta$ as
\begin{equation}
\label{eq:ray_integral_expansion}
\phi(\eta) - \phi(0) \;=\; -\eta \vg^\top \vu + \frac{\eta^2}{2} \vu^\top \bar{\mH}_{\eta, \vu} \vu
\;=\; -s A + \frac{s^2}{2} B,
\end{equation}
where we define the \emph{alignment score} $A := \vg^\top \vu$ and the \emph{curvature load} $B := \tfrac{\eta}{2} \vu^\top \bar{\mH}_{\eta, \vu} \vu$.
Then, the realized update yields the exact expression
\begin{equation}\tag{{\color{black!30}$\diamondsuit$}}
\label{eq:ray_stability_boundary}
\renewcommand{\boxed}[1]{\begingroup\setlength{\fboxsep}{4pt}\colorbox{lightgray!15}{\hspace*{.5em}
\ensuremath{#1}\hspace*{.5em}}\endgroup}
\boxed{
  \displaystyle
\Delta \gL \;:=\; \gL(\vtheta - \eta \vu) - \gL(\vtheta) \;=\; -\eta A + \eta B, 
\quad\text{and}\quad
\zeta \;:=\; \frac{B}{A}.
}
\end{equation}
We define the \textbf{realized edge of stability} $\zeta$ as the ratio of the curvature load to the alignment score.
It is helpful to notice that, by the chain rule, the alignment score is \emph{exactly} the learning 
rate-compensated loss drop rate for the full-batch update~\citep{lee2025greedy}:
\begin{equation}
\label{eq:alignment_power_chain_rule}
A \;=\; \vg^\top \vu \;=\; \nabla_{\vtheta} \gL^\top \left( - \frac{\mathrm d\vtheta}{\eta \mathrm dt} \right) \;=\; -\frac{1}{\eta} \frac{\mathrm d \gL}{\mathrm dt} \;=\; - \frac{\mathrm d \gL}{\mathrm ds}.
\end{equation}
In other words, the alignment score $A$ measures the actual \emph{learning work} done by the optimizer to decrease the loss along the update direction $\vu$, with a time granularity (sampling interval) $\eta \mathrm dt = \mathrm ds$.
Similarly, the curvature load $B$ measures the reduction in the loss drop rate due to the curvature of the loss landscape along the update direction $\vu$:
\begin{equation}
\label{eq:curvature_load_chain_rule}
B \;=\; \frac{\eta}{2} \vu^\top \bar{\mH}_{\eta,\vu} \vu \;=\; A + \frac{\gL(\vtheta - \eta \vu) - \gL(\vtheta)}{\eta} \;=\; A + \frac{\Delta \gL}{\eta} \;=\; - \frac{\mathrm d \gL}{\mathrm ds} - \left(- \frac{\Delta \gL}{\Delta s} \right),
\end{equation}
where $\Delta s = \eta \Delta t = \eta$ (since $\Delta t = 1$ for the iteration indices).
Borrowing a physics analogy, if we casually refer to the loss drop rate $A$ as the \emph{learning power}, we can call the curvature load $B$ the \emph{learning power dissipation} by the loss landscape along the update direction $\vu$.
Therefore, $\zeta$ is a unit-free ratio of the learning power dissipation to the learning power output of the optimizer $(\mP, Q)$.
In this sense, we may call the realized EoS $\zeta$ the \textbf{learning impedance} of the task.

To better understand the relationship between the realized $\zeta$ and the relative edge of stability $\chi$, let us define a few more quantities and decompose the realized EoS $\zeta$ into factors involving $\chi$.
The directional smoothness from previous works~\citep{lee2023ias,mishkin2024directional,islamov2026noneuclidean}, $D^{\|\cdot\|}$, is a normalized curvature load, written in our notation as $D^{\|\cdot\|_{\mP}} := 2B / (\eta \|\vu\|_{\mP}^2)$.
Recall that the maximum sharpness $S$ in (\ref{eq:marginal_stability_condition}) is defined as $S = \lambda_{\max}(\mP^{-1/2} \mH \mP^{-1/2}) = \max_{\vu \neq 0} \vu^\top \mH \vu / \|\vu\|_{\mP}^2$.
If the Hessian $\mH$ remains almost constant along the ray, then $\bar \mH_{\eta,\vu} \simeq \mH$ implies $S \simeq \max_{\vu \neq 0} D^{\|\cdot\|_{\mP}}$.
Hence, $D^{\|\cdot\|_{\mP}} \le S$.
Therefore, we can define a \emph{spatial participation factor} $\varsigma := D^{\|\cdot\|_{\mP}} / S$, a unit-free measure of how much of the maximally available sharpness is actuated by the curvature along the update direction $\vu$ taken by the optimizer.
Its value becomes 1 when we take the worst-case direction $\vu = \operatorname{span}(\vv_{\max})$, and is lower when we take other directions.

Also, recall that $\Gamma$ in (\ref{eq:marginal_stability_condition}) is the gain of the gradient filter $Q$ at the escape frequency of its root locus.
In other words, at this escape frequency, the optimizer applies the update $\vu = \mP^{-1} (\Gamma \vg)$.
Equivalently, $\vg = \Gamma^{-1} \mP \vu$.
If the update realizes the escape mode of Proposition~\ref{prop:edge_preconditioner}, then $\vu = \mP^{-1}(\Gamma \vg)$ and the learning power takes the form $A = (\Gamma^{-1} \mP \vu)^\top \vu = \|\vu\|_{\mP}^2 / \Gamma$.
Therefore, we can define a \emph{temporal calibration factor} $\tau := A / (\|\vu\|_{\mP}^2 / \Gamma) = \Gamma A / \|\vu\|_{\mP}^2$, a unit-free measure of how closely the actual learning power agrees with the power predicted by that escape mode.
Its value is 1 when the filtered stream is in phase with $\Gamma \vg$ in the $\mP$-inner product, and is identically $1$ for vanilla GD ($\Gamma = 1$, $\vu = \vg$).

Using these factors, we can now \emph{exactly} decompose the realized EoS $\zeta$ into the following factors:
\begin{equation}\tag{{\color{black}$\heartsuit$}}
\renewcommand{\boxed}[1]{\begingroup\setlength{\fboxsep}{4pt}\colorbox{lightgray!15}{\hspace*{.5em}\ensuremath{#1}\hspace*{.5em}}\endgroup}
\label{eq:optimizer_aware_eos_factor}
\boxed{
\displaystyle
\zeta
\;=\; \frac{\eta S \Gamma}{2} \cdot \frac{\vu^\top \bar{\mH}_{\eta,\vu} \vu}{S \|\vu\|_{\mP}^2} \cdot \frac{\|\vu\|_{\mP}^2}{\Gamma \vg^\top \vu}
\;=\; \chi \frac{\varsigma}{\tau}
\;=\; (\text{relative EoS}) {\times} \frac{(\text{spatial participation})}{(\text{temporal calibration})},
}
\end{equation}
Therefore $\zeta = \chi$ if and only if $\varsigma = \tau$.
The factors $\varsigma$ and $\tau$ account for the spatial and temporal discrepancies between the actual update and the binding mode of the available certificate, and their equivalence is not trivially guaranteed.
Refer to Corollary~\ref{cor:coincidence} and Appendix~\ref{appx:proof:coincidence} for the exact condition and the GD reduction $\tau \equiv 1$.

\section{Coincidence Theorem}
\label{sec:discussion}
We now delve into the implications of the new formulation (\ref{eq:optimizer_aware_eos_factor}) in more detail.
We are interested in the conditions under which the two values $\zeta$ and $\chi$ are equal and when they become different.
The following corollary is the exact content of~(\ref{eq:optimizer_aware_eos_factor}).

\begin{corollary}[Coincidence]
\label{cor:coincidence}
\textnormal{[\texttt{\hyperref[appx:proof:coincidence]{proof}}]}
Fix an iteration and assume $\eta, S, \Gamma > 0$, $\vu \neq \vzero$, and $A = \vg^\top \vu \neq 0$.
Then~(\ref{eq:optimizer_aware_eos_factor}) is an identity, and
\begin{equation}
\label{eq:coincidence_condition}
\zeta \;=\; \chi
\quad\Longleftrightarrow\quad
\varsigma \;=\; \tau
\quad\Longleftrightarrow\quad
\vu^\top \bar{\mH}_{\eta, \vu} \vu \;=\; S \Gamma \, \vg^\top \vu.
\end{equation}
The separate conditions $\varsigma = 1$ and $\tau = 1$ are sufficient for coincidence and not necessary: they impose two scalar constraints where~(\ref{eq:coincidence_condition}) imposes one.
In particular, vanilla GD ($\mP = I$, $Q \equiv 1$) has $\Gamma = 1$ and $\tau \equiv 1$ at every step, so $\zeta = \chi \varsigma$ and coincidence reduces to occupation $\varsigma = 1$.
At the realized edge $\zeta = 1$ one must have $\chi = \tau / \varsigma$.
\end{corollary}

The proof is the cancellation in Appendix~\ref{appx:proof:coincidence}.
In plain language, the original edge $\chi$ equals the realized edge $\zeta$ exactly when spatial participation $\varsigma$ matches temporal calibration $\tau$.
For GD, since its $\tau$ is fixed to unity, the systematic violation of $\chi > 1$ immediately translates to $\varsigma < 1$.
We observe that Adam, AdamW, and Heavy Ball in Section~\ref{sec:observation} have both $\chi$ and $\zeta$ near $1$.
This translates to $\varsigma \approx \tau$.

The following corollary isolates the one special case in which coincidence is structural rather than accidental, namely $\varsigma = \tau = 1$ rather than a mere cancellation $\varsigma = \tau$.

\begin{corollary}[Coincidence on a mode-locked quadratic system]
\label{cor:mode_locked_coincidence}
\textnormal{[\texttt{\hyperref[appx:proof:mode_locked]{proof}}]}
Assume the following over the memory horizon of the filter $Q$.
(Q1) The loss is quadratic over the memory horizon: $\vg_s = \mH(\vtheta_s - \vtheta_\star)$ and therefore $\bar{\mH}_{\eta, \vu} = \mH$.
(Q2) The preconditioner is frozen: $\mP_s \equiv \mP \succ 0$.
(Q3) The trajectory is locked to the sharpest preconditioned eigenmode: $\vtheta_s - \vtheta_\star = x_s \mP^{-1/2} \vv_{\max}$ for scalars $x_s$, where $x_t \neq 0$.
(Q4) The root locus of $Q$ escapes at the Nyquist frequency: $\Gamma = Q(-1)$, and the locked mode is in period-2 oscillation, $x_{t-k} = (-1)^k x_t$.
Then the alignment score is automatically nondegenerate, $A = \Gamma S^2 x_t^2 > 0$, and $\varsigma = \tau = 1$, and therefore the realized EoS $\zeta$ is exactly the relative edge of stability $\chi$: $\zeta = \chi = \eta S Q(-1)/2$.
Dropping the period-2 part of (Q4), even while keeping $\Gamma = Q(-1)$, preserves $\varsigma = 1$ but gives
\begin{equation}
\label{eq:mode_locked_tau}
\tau \;=\; \frac{\Gamma x_t}{(Q * x_{\le t})_t},
\end{equation}
which equals 1 only when the filtered stream is in phase with the gradient at the escape gain.
\end{corollary}

The proof is provided in Appendix~\ref{appx:proof:mode_locked}.
Corollary~\ref{cor:mode_locked_coincidence} reveals assumptions that are used to derive the original boundary $\chi \le 1$ in previous works~\citep{cohen2021gradient,cohen2023adaptive}.
In this worst-mode-locked quadratic regime, our realized EoS $\zeta \le 1$ of (\ref{eq:ray_stability_boundary}) is exactly the original EoS $\chi \le 1$ of (\ref{eq:marginal_stability_condition}).
Our formulation generalizes the previous works in this precise sense.

Each of the assumptions~(Q1)--(Q4) fails in an identifiable way, and the observations of Section~\ref{sec:observation} split across those failure types.
The decomposition of (\ref{eq:optimizer_aware_eos_factor}), therefore, not only serves as an evidence that $\zeta$ generalizes $\chi$, but also a useful diagnostic toolkit for analyzing the edge of stability:
(Q1) fails whenever the third-order term along the ray is non-negligible, so that $\bar{\mH}_{\eta,\vu} \neq \mH$;
(Q2) fails for adaptive preconditioners, whose $\mP_t$ drifts over the memory horizon of $Q$;
(Q3) fails whenever the update actuates a non-maximal direction, \emph{including} vanilla GD, where $\vu = \vg$;
and (Q4) fails for any filter whose root locus escapes away from $z = -1$, or whose locked mode is not in period-2 oscillation.
We can categorize the observation results in Section~\ref{sec:observation} into the following cases: GD is~(Q3) with $\tau \equiv 1$ ($\chi = 1.124 > 1$, $\zeta = 0.983 \approx 1$);
Adam, AdamW, and Heavy Ball are near coincidence after $\Gamma$-scaling ($\chi$ at $0.960$, $0.971$, and $0.949$ against $\zeta$ at $0.964$, $0.919$, and $0.983$, respectively);
AdaGrad is~(Q2)--(Q3) ($\chi = 21.070 \gg 1$, yet $\zeta = 1.002 \approx 1$);
QHM mixes a moderate $\chi$ offset with $\zeta = 1.000$;
and the two-pole cascades are~(Q4), with $\chi = 2.646 > 1$ and no $\zeta$ reaches the edge.
This identifies the actual causes of the offsets as elaborated in Section~\ref{sec:extension}.

We now record what Corollary~\ref{cor:coincidence} implies for the available stability limit: $\chi$ determines $\zeta$ only through the ratio $\varsigma / \tau$, which the gauges $(S, \Gamma)$ do not control.

\begin{theorem}[Insufficiency of $\chi$ for the realized edge]
\label{thm:nonseparability}
\textnormal{[\texttt{\hyperref[appx:proof:nonseparability]{proof}}]}
Assume $\eta, S, \Gamma > 0$, $\vu \neq \vzero$, and $A \neq 0$.
By Corollary~\ref{cor:coincidence}, $\zeta = \chi$ if and only if $\varsigma = \tau$.
Consequently $(\eta, S, \Gamma)$ determine $\zeta$ if and only if they determine the ratio $\varsigma / \tau$.
For vanilla GD ($\mP = I$, $Q \equiv 1$) one has $\tau \equiv 1$ and $\zeta = \chi \varsigma$, so $\chi$ determines $\zeta$ if and only if occupation $\varsigma$ is a function of $(\eta, S)$ alone.
This already fails on any frozen quadratic in dimension at least $2$ whose gradient is not an eigen-direction of $\mH$: then $\varsigma$ is the Rayleigh ratio $\vg^\top \mH \vg / (S \|\vg\|^2)$ and is not determined by $S = \lambda_{\max}(\mH)$.
\end{theorem}

\begin{figure}[t]
\centering
\begin{minipage}{\linewidth}
\centering
\includegraphics[width=\linewidth]{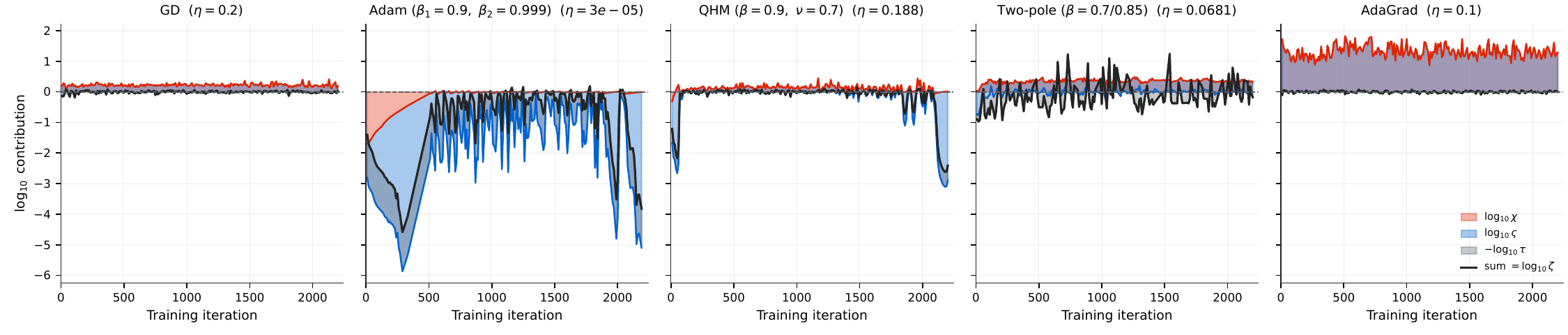}
\vspace{-2.2em}
\caption*{%
\small
    {(a)}\, Decomposition of the realized edge of stability $\zeta$ from a single experiment.
}
\end{minipage}
\begin{minipage}{\linewidth}
\centering
\includegraphics[width=\linewidth]{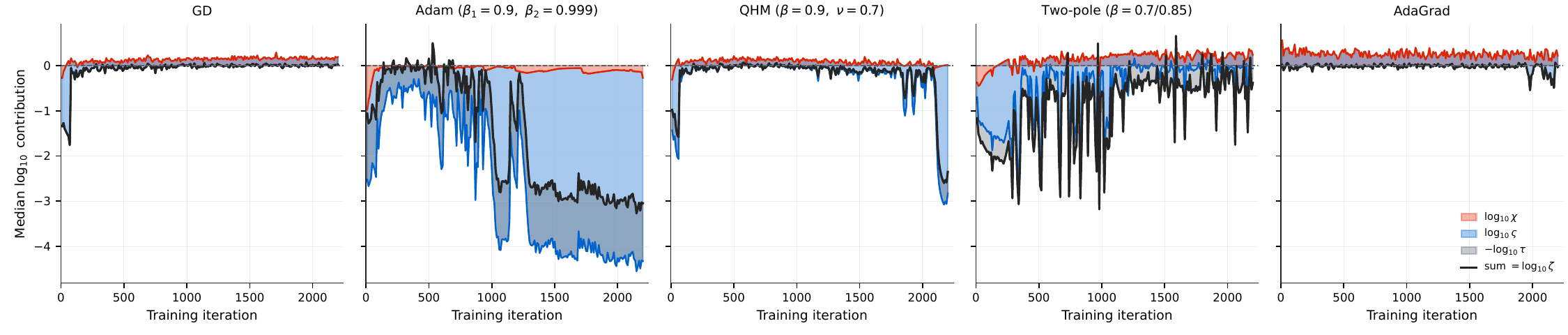}
\vspace{-2.2em}
\caption*{%
\small
    {(b)}\, Decomposition of optimizer-family-wise contributions to $\zeta$, median from an LR sweep of size 5.
}
\end{minipage}

\vspace{-.7em}
\caption{%
\small
\textbf{EoS stack plot.}
Decomposition of the realized EoS $\zeta$ into factors in log-scale, involving the (relative) edge of stability $\chi$ ({\color{red!50} red} area and line), the spatial participation factor $\varsigma$ ({\color{blue!50} blue} area and line), and the temporal calibration factor $\tau$ ({\color{white!30!black} grey} overlay and black line).
Note that $\log \tau$ operates in the opposite direction by subtraction.
}
\label{fig:decomposition_of_learning_impedance}
\vspace{-1em}
\end{figure}

The proof of Theorem~\ref{thm:nonseparability} is the identity of Corollary~\ref{cor:coincidence} together with the GD reduction $\tau \equiv 1$, detailed in Appendix~\ref{appx:proof:nonseparability}.
It is important to mention that this result does not falsify previous works' contributions on formulating the EoS~\citep{cohen2021gradient,cohen2023adaptive}, but rather we circumscribe the scope of their claims to the case where the update direction is fixed to the fixed maximum curvature, which is condition (Q3) of Corollary~\ref{cor:mode_locked_coincidence}.
Violation of this condition implies that the optimizer is more than just a prescribed algorithmic procedure, but rather an active decision-maker that chooses the edge it wants to occupy among the available options provided by the current loss landscape.

\section{Discussion}
\label{sec:extension}
Previous formulations of the edge of stability~\citep{cohen2021gradient,cohen2023adaptive,islamov2026noneuclidean} and its generalized version in Proposition~\ref{prop:edge_preconditioner} predict available worst-case marginality $\chi$.
Our realized EoS $\zeta$ reveals whether and how this available margin is actually trod upon by the optimizer.
This distinction offers a new way to understand the training dynamics of different optimizer families.

\subsection{Diagnostic Applications}
\label{sec:discussion:diagnostic_applications}

If we take the logarithm from (\ref{eq:optimizer_aware_eos_factor}), we can decompose the realized EoS $\zeta$ into additive factors:
\begin{equation}
\label{eq:log_decomposition}
\log \zeta = \log \chi + \log \varsigma - \log \tau.
\end{equation}
This simple additive relationship holds for every iteration, allowing us to draw an \textbf{EoS stack plot}, as in Figure~\ref{fig:decomposition_of_learning_impedance}, to visualize how each factor contributes to the realized EoS $\zeta$ during the training process.
We can use this stack plot as a new diagnostic tool to study training dynamics online, for various instances and families of optimizers.

The first thing we notice from the EoS stack plot in Figure~\ref{fig:decomposition_of_learning_impedance} is that the spatial participation factor $\varsigma$ compensates for $\chi$ even without a nontrivial preconditioner.
Vanilla GD has $\tau \equiv 1$, so the stack is $\log \zeta = \log \chi + \log \varsigma$.
The compensation from $\chi = 1.124$ to $\zeta = 0.983$ is solely due to $\varsigma < 1$.
The same happens to AdaGrad's extreme offset $\chi \approx 21.1$, which is mostly compensated by $\varsigma$, resulting in $\zeta \approx 1$.
In other words, the optimizer often selects updates $\vu$ in nonmaximal sharpness directions, and GD already does so.
This explains why the previous edge of stability formulation based on the worst-mode update $\vu_{\max}$ was not sufficient to predict the actual EoS $\zeta$.
Another interesting observation is that the temporal calibration factor $\tau$ can act in either direction, depending on the type of gradient filter $Q$.
For Adam and QHM, $\log \tau$ is mostly negative, but for more complicated filters like two-pole filters, $\log \tau$ can be positive, pulling the dynamics away from the edge of stability.
Adam's near-coincidence $\chi \approx \zeta \approx 1$ is therefore $\varsigma \approx \tau$ after $\Gamma$ is included in $\chi$, not a restoration of occupation $\varsigma = 1$.
This single observation already explains the main cause of the family-dependent offsets in the edge of stability $\chi$ and allows us to visualize how different optimizer families compensate for this effect and actively shape their training dynamics online.

\begin{table}[t]
\begin{minipage}{0.62\linewidth}
\centering
\caption{\small
\textbf{Computation requirements.}
For one full-batch optimizer step (Adam, params=777k, batch=50k, 8-step Lanczos, 1$\times$ RTX 5090).
$\zeta$ requires no HVPs, reducing computation time by 95.3\%.
}
\label{tab:decomposition-component-cost}
\resizebox{\linewidth}{!}{%
\begin{tabular}{@{}llccc@{}}
    \toprule
    \textbf{Component} & \textbf{Calculation formula} & \textbf{HVPs} & \textbf{Time (ms)} & \textbf{Peak mem.\ (MiB)} \\
    \midrule
    $\chi=\eta S\Gamma / 2$ & $k$-step Lanczos sharpness & 8 & 110.15 $\pm$ 2.56 & 767.3 \\
    $\varsigma=2B/(\eta S \|\vu\|_\mP^2)$ & $k$-step Lanczos sharpness & 8 & 114.84 $\pm$ 1.15 & 767.3 \\
    $\tau=\Gamma A/\|\vu\|_\mP^2$ & gradient contractions & 0 & 5.22 $\pm$ 0.00 & 305.4 \\
    \midrule
    $\zeta=B/A$ & gradient contractions & 0 & 5.20 $\pm$ 0.01 & 305.4 \\
    \bottomrule
\end{tabular}
}
\end{minipage}
\hfill
\begin{minipage}{0.36\linewidth}
\centering
\includegraphics[width=\linewidth]{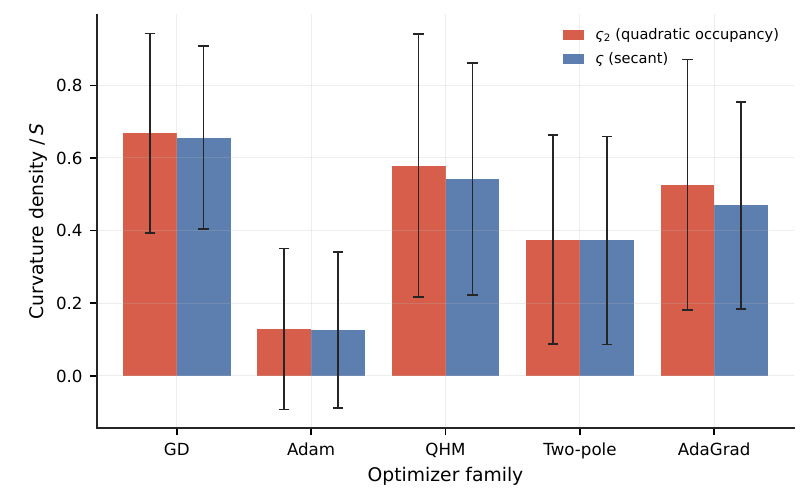}
\vspace{-1.8em}
\captionof{figure}{\small
$\varsigma$ vs. $\varsigma_2$ for optimizers.}
\label{fig:comparison}
\end{minipage}
\end{table}

\subsection{Cost Benefits}
\label{sec:discussion:approximation_cost_and_error_bound}

One additional benefit of realized EoS $\zeta$ is that it does not require heavy calculation.
Since
$
\zeta = \tfrac{B}{A} = \tfrac{A + \Delta \gL / \eta}{A} = 1 + \tfrac{\Delta \gL}{\eta \vg^\top \vu},
$
we can compute $\zeta$ directly from the current gradient $\vg$, the chosen update direction $\vu$, and the resulting change in loss $\Delta \gL$ at each training iteration.
In comparison, previous definition of the EoS involves a maximizer on preconditioned Hessian~\citep{cohen2021gradient,cohen2023adaptive,islamov2026noneuclidean}, which is not only computationally expensive, requiring multiple Hessian-vector products (HVPs), but also forces us to use approximate measures.
Table~\ref{tab:decomposition-component-cost} summarizes the computation costs for each component of (\ref{eq:log_decomposition}), all measured from the same CIFAR-10 training using the 200-5-tanh model in the initial report~\citep{cohen2021gradient}.
We highlight that our formulation is both cost-effective and \emph{exact} for full-batch optimization.

\subsection{Effect of Higher-Order Terms}
\label{sec:discussion:effect_of_higher_order_terms}

Similar to the original edge of stability~\citep{cohen2021gradient,cohen2023adaptive,islamov2026noneuclidean}, Figure~\ref{fig:observation} shows how actual training trajectories are attracted to, oscillate around, and are eventually pushed away from the edge of stability $\zeta = 1$.
As in previous works on EoS~\citep{damian2023self,mulayoff2026nonlinear}, we can derive sufficient conditions for the stabilization mechanism to take effect.
Since our main focus is to resolve the optimizer-dependent offsets found in the empirical EoS, we defer this analysis to Appendix~\ref{sec:third_order}.
In our formulation of $\zeta$, the higher-order terms are absorbed into the weighted secant Hessian $\bar{\mH}_{\eta,\vu}$ and thus into the spatial participation factor $\varsigma$.
If we want to isolate the effect of the higher-order terms, we can define the second-order-only spatial participation factor $\varsigma_2 := \vu^\top \mH \vu / (S\|\vu\|_{\mP}^2)$ and compare it with $\varsigma$, as shown in Figure~\ref{fig:comparison}.
Although higher-order terms are essential for maintaining stability beyond the edge of stability, their contribution to the magnitudes of the realized EoS $\zeta$ is small in the studied settings.

\subsection{Limitations}
\label{sec:discussion:limitations}

This work focuses on full-batch optimization, where the theory is exactly realized in practice.
In mini-batch optimization, the terms $A = \vg^\top \vu$, $B = \eta \vu^\top \bar\mH \vu / 2$, and their ratio $\zeta = B / A$ become stochastic due to changes in the instantaneous loss landscape $\gL$ along the trajectory.
This makes the analysis more challenging and complicates the interpretation of the ratio $\zeta$.
Previous works treat the incorporation of stochastic scenarios as an independent problem, which is sometimes referred to as the edge of stochastic stability (EoSS) \citet{andreyev2024edge,andreyev2026momentum}.
We likewise leave the extension of our framework to the stochastic setting for future work.

\section{Conclusion}
\label{sec:conclusion}
The edge of stability has been studied primarily as an effect related to the maximum curvature mode of the loss landscape.
This turns out to be only half of the story.
The other half is the effect of time-varying stateful optimizers that actually take the updates, given a loss landscape and gradient history.
Neglecting this aspect leads to optimizer-dependent systematic offsets in estimating the edge of stability, which are largely eliminated in our framework on the \emph{realized} edge of stability.
Decomposing the realized EoS into factors results in diagnostic tools that provide a richer understanding of training dynamics under gradient-based optimizers.

\subsection*{AI use statement}

In this work, we used AI assistants solely to assist with writing code and in assisting with mathematical proofs.
All mathematical claims or proofs augmented with AI assistance were thoroughly reviewed, verified, and rewritten as needed by the human authors to ensure correctness and rigor.
We take full responsibility for the final content of this work, including all text, code, and mathematical results.

\subsection*{Reproducibility statement}

We have made every effort to ensure that our results are reproducible.
Sufficient explanations of our methods and setups are provided in Appendix~\ref{sec:experiment}.
All figures presented in the paper can be fully reproduced using the code and data included in the supplementary material.



\bibliography{main}
\bibliographystyle{iclr2027_conference}

\appendix

\section{Related Work}
\label{sec:survey}
\paragraph{Edge of stability.}
It has been known that the learning rate constrains the curvature of the loss landscape during the course of training~\citep{xing2018walk,jastrzebski2019relation,gilmer2022loss}.
Later, \citet{cohen2021gradient} conducted an extensive study on this phenomenon and first coined the term \emph{edge of stability} to describe the attractive behavior of the maximum Hessian eigenvalue $\lambda_{\max}(\mH)$ towards a value reciprocal to the learning rate $\lambda_{\max} \approx 2 / \eta$ under full-batch gradient descent.
Similar phenomena have been reported under different optimization settings~\citep{lyu2022understanding,cohen2023adaptive,damian2023self,agarwala2023second,zhu2023understanding,andreyev2024edge,chen2024stability,cohen2025central,islamov2026noneuclidean,andreyev2026momentum,litman2026origin}.
This includes full gradient descent~\citep{cohen2021gradient}, mini-batch gradient descent~\citep{andreyev2024edge}, momentum-based optimizers~\citep{cohen2023adaptive,andreyev2026momentum}, and adaptive method~\citep{cohen2023adaptive} such as RMSProp~\citep{tieleman2012lecture} or Adam~\citep{kingma2015adam}.
Following these lines of work, we primarily focus on the full-batch settings where the analysis is more tractable and the effect is more pronounced.
In case of stochastic mini-batch gradient descent, \citet{andreyev2024edge,andreyev2026momentum} empirically found that the edge of stability extends beyond determinism and applies to curvature of the batch statistics.

\paragraph{Mechanisms: catapult, progressive sharpening, self-stabilization, and bifurcation.}
Existing explanations for the EoS phenomenon can be categorized into four mechanisms:
(1) The \emph{catapult}~\citep{lewkowycz2020large}, described prior to the formalization of the edge of stability, explains how a large learning rate drives a sudden loss spike that initiates the reduction of the curvature down to the edge of stability.
(2) \citet{cohen2021gradient,wang2022analyzing} demonstrate the opposite, called \emph{progressive sharpening}, where curvature initially increases up to the EoS bound.
The EoS phenomenon also implies the existence of a coincidental stabilization mechanism that consistently reduces the curvature to the edge, causing it to hover just over the theoretical stability bound.
(3) \citet{damian2023self} attribute this \emph{self-stabilization} to curvature compensation by third-order derivatives.
(4) \citet{jastrzebski2020breakeven,song2023trajectory} identify a period-2 \emph{flip bifurcation} in the discrete dynamical system governed by gradient descent.
These four mechanisms form the basis of the EoS phenomenon, and many follow-up works have sought to explain them using simpler models.
\citet{arora2022understanding} use a two-timescale decomposition;
\citet{ma2022multiscale} employ a one-dimensional model;
\citet{chen2023beyond} analyze two-step fixed points that extract period-2 motions;
\citet{kreisler2023gfs} study scalar linear networks; and
\cite{hofmann2026map} use a simple quartic model.
Continuous-time reductions are also explored, including central flows~\citep{cohen2025central}, Edge Flow~\citep{marion2026edge}, and Rod Flow~\citep{regis2026rod,regis2026adam}.
In structured models, EoS is used as an implicit-bias mechanism for diagonal linear networks~\citep{even2023sgd}, logistic regression~\citep{wu2023implicit,wu2024large}, threshold neurons~\citep{ahn2023learning}, and deep matrix factorization~\citep{ghosh2025learning}; \citet{wang2025good} unify EoS, balancing, and catapult as large-learning-rate biases resulting from regularity, and \citet{jiang2025understanding} track the NTK at EoS.
None of these works reformulate the dependence of the EoS $\chi = \eta S\Gamma / 2$ on the maximum sharpness direction $\vv_{\max}$ to incorporate actual executed updates carried out \emph{by the optimizers}, which is the gap we address.

\paragraph{Directional smoothness.}
We are mainly inspired by prior work examining how directional smoothness influences the stability bound of full-batch gradient descent.
\citet{ahn2022unstable} formalize unstable convergence in terms of directional smoothness and relative progress, finding structured oscillations rather than divergence past $2/\eta$.
\citet{lee2023ias} replace $\lambda_{\max}$ with an interaction-aware sharpness (batch-gradient--Hessian coupling), and report that this quantity hovers around a concentration measure under mini-batch SGD.
\citet{mishkin2024directional,islamov2026noneuclidean} investigate the directional smoothness of the executed step as a convergence diagnostic.
Their works are analogous to our realized EoS $\zeta = B/A$ along an update direction $\vu$; however, their quantities are not directly representative of the edge of stability, as they resort to the maximum sharpness instead of the executed update.
Furthermore, previous works have not accounted for general filtered optimizers.
Therefore, our work extends these works rather than competing with them.

\paragraph{Gradient-based optimizers as filters.}
Momentum methods~\citep{polyak1964some,nesterov1983method} are linear time-invariant (LTI) filters applied to the gradient stream.
\citet{qian1999momentum} identified Heavy Ball~\citep{polyak1964some} as a discrete-time low-pass filter.
Control systems theory makes this viewpoint precise: first-order methods are modeled as LTI plants in feedback with gradient stream signals $\vg$.
These systems can be analyzed using transfer functions, Bode plots, and integral quadratic constraints~\citep{lessard2016iqc,hu2017control}, and can be designed as higher-order filters such as triple momentum~\citep{vanscoy2018triple}.
\citet{gitman2019momentum} treat Heavy Ball~\citep{polyak1964some}, Nesterov's momentum~\citep{nesterov1983method}, and quasi-hyperbolic momentum~\citep{ma2018quasi} as a single IIR family, and map stability regions in $(\beta,\nu)$.
\citet{lee2024grokfast} design a two-timescale EMA filter with preconditioner~\citep{kingma2015adam,loshchilov2019decoupled}.
\citet{li2025frequency} apply the $z$-transform to momentum methods and interpret the coefficients as a (time-varying) frequency response.
\citet{lee2025greedy} proposes an adaptive technique to determine the optimal filter coefficients, e.g., the momentum hyperparameter $\beta$, further developing this signal processing perspective.
In summary, momentum-based optimizers can be understood as filters: GD acts as an identity filter, Heavy Ball and Nesterov's momentum are one-pole filters, QHM is a one-tap pole-zero filter, and two-timescale EMAs~\citep{ghosh2025learning} such as Grokfast~\citep{lee2024grokfast} function as cascaded two-pole filters.
The latter serve as practical counterparts of $Q$, whose root locus does not necessarily escape at the Nyquist frequency $z = -1$.
We draw extensively on this perspective to study the EoS phenomenon under general filtered optimizers.

\section{Proofs}
\label{sec:proofs}
\subsection{The Mode Reduction}
\label{appx:proof:mode_reduction}

Throughout this subsection we work in the \emph{frozen quadratic} regime on which the original edge-of-stability arguments are based: over the memory horizon of the filter the loss is $\gL(\vtheta) = \tfrac12 (\vtheta - \vtheta_\star)^\top \mH (\vtheta - \vtheta_\star)$ with $\mH$ constant and symmetric, the preconditioner is frozen at $\mP_t \equiv \mP \succ 0$, and the
filter is time invariant, $Q_t \equiv Q$.
Let $(\lambda_i, \vv_i)$ be the eigenpairs of $\mP^{-1/2} \mH \mP^{-1/2}$ with $\{\vv_i\}$
orthonormal, and define the \emph{whitened} coordinates
\begin{equation}
\label{eq:appx_whitened_coordinate}
x_t^{(i)} \;:=\; \vv_i^\top \mP^{1/2} (\vtheta_t - \vtheta_\star) .
\end{equation}
Two remarks on~(\ref{eq:appx_whitened_coordinate}).
The factor $\mP^{1/2}$ applies before the mode decomposition, and the offset $\vtheta_\star$ makes the gradient proportional to the coordinate.
The shorthand $x_i = \vv_i \cdot \vtheta$ used below~(\ref{eq:eigencomponent_dynamics}) incorporates both effects for brevity.

\begin{lemma}[Exact decoupling]
\label{lem:appx_decoupling}
Under the frozen quadratic regime, the update rule~(\ref{eq:gradient_filter}) is equivalent to the family of decoupled scalar recursions
\begin{equation}
\label{eq:appx_scalar_recursion}
x^{(i)}_{t+1} \;=\; x^{(i)}_t - \eta \lambda_i \, (Q * x^{(i)}_{\le t})_t ,
\qquad i = 1, \dots, d ,
\end{equation}
which is~(\ref{eq:eigencomponent_dynamics}) with the coordinates~(\ref{eq:appx_whitened_coordinate}).
\end{lemma}

\begin{proof}
Since $\vg_t = \mH (\vtheta_t - \vtheta_\star)$, inserting $\mP^{\pm 1/2}$ gives $\mP^{-1/2} \vg_t = (\mP^{-1/2} \mH \mP^{-1/2}) \, \mP^{1/2} (\vtheta_t - \vtheta_\star)$, so $\vv_i^\top \mP^{-1/2} \vg_t = \lambda_i x_t^{(i)}$.
Applying $\mP^{1/2}$ to $\vtheta_{t+1} - \vtheta_\star = \vtheta_t - \vtheta_\star - \eta \mP^{-1} (Q * \vg_{\le t})_t$ and using $\mP^{-1/2} (Q * \vg_{\le t})_t = (Q * \mP^{-1/2}\vg_{\le t})_t$
(which holds because the convolution acts on the time index and $\mP$ is constant),
we obtain $\mP^{1/2}(\vtheta_{t+1} - \vtheta_\star) = \mP^{1/2}(\vtheta_t - \vtheta_\star) - \eta (Q * \mP^{-1/2}\vg_{\le t})_t$.
Projecting onto $\vv_i$ yields~(\ref{eq:appx_scalar_recursion}).
Since no two modes interact, the reduction is exact rather than an approximation.
\end{proof}

From here on we can fix one mode without loss of generality, drop the index, and write the \emph{loop gain}
\begin{equation}
\label{eq:appx_loop_gain}
k \;:=\; \eta \lambda \;>\; 0 ,
\end{equation}
assuming $\lambda > 0$.
Modes with $\lambda \le 0$ do not contribute to unit-circle crossing for $k > 0$, and need not be considered.

\subsection{Proof of Proposition~\ref{prop:edge_preconditioner}}
\label{appx:proof:eos_prop}

\begin{proof}
We first record the two standing hypotheses that the statement of the proposition leaves implicit, but generallly holds true for practical filters.
Write $Q(z)$ in a rational form $Q(z) = b(\xi) / a(\xi)$ with $\xi := z^{-1}$, where $a, b$ are coprime real polynomials normalized so that $a(0) = 1$.
\begin{enumerate}
\item[(H1)] $Q$ is causal and internally stable: every pole of $Q$ lies in $|z| < 1$, equivalently every root of $a$ lies in $|\xi| > 1$.
\item[(H2)] $Q(1)$ is finite and nonzero; by the dc normalization of Section~\ref{sec:preliminary}, $Q(1) = 1$ always.
\end{enumerate}
Without (H1) the filter is unstable on its own and no gain $k$ stabilizes the mode, so the phrase ``first crossing'' has no meaning.
Without (H2) the branch of the locus that starts at $z = 1$ does not move into the disc.

\medskip\noindent\textit{The characteristic polynomial.}
Taking $z$-transforms of~(\ref{eq:appx_scalar_recursion}) with zero initial data gives $z X(z) = X(z) - k Q(z) X(z)$, hence~(\ref{eq:eigencomponent_transfer_function}),
\begin{equation}
\label{eq:appx_char_z}
z - 1 + k Q(z) \;=\; 0 .
\end{equation}
Multiplying~(\ref{eq:appx_char_z}) by $z^{-1} a(z^{-1})$ and substituting $\xi= z^{-1}$ clears both the pole at $z = 0$ and the denominator of $Q$, yielding a characteristic polynomial
\begin{equation}
\label{eq:appx_char_zeta}
c_k(\xi) \;:=\; (1 - \xi)\, a(\xi) \;+\; k\, \xi\, b(\xi) \;=\; 0 .
\end{equation}
Because $|z| < 1 \iff |\xi| > 1$, the mode is asymptotically stable exactly when every root of $c_k$ lies strictly outside the closed unit disc.
Roots of $c_k$ move continuously with $k$ on the Riemann sphere, and the exterior $\{|\xi| > 1\} \cup \{\infty\}$ is open, so a root can enter the closed unit disc only by passing through the circle $|\xi| = 1$.
A root escaping to $\xi = \infty$, which happens when the leading coefficient of~(\ref{eq:appx_char_zeta}) degenerates, is a deadbeat pole at $z = 0$ and never destabilizes.

\medskip\noindent\textit{Step 1: the locus starts inside.}
At $k = 0$ we have $c_0(\xi) = (1 - \xi) a(\xi)$, whose roots are $\xi = 1$ together with the roots of $a$, which all lie in $|\xi| > 1$ by (H1).
The root at $\xi = 1$ is simple because $a(1) \neq 0$ by (H2).
Since $\partial_\xi c_k |_{(1,0)} = -a(1)$ and $\partial_k c_k |_{(1,0)} = b(1)$, the implicit function theorem gives a branch $\xi(k)$ with
\begin{equation}
\label{eq:appx_locus_departure}
\left. \frac{d\xi}{dk} \right|_{k=0} \;=\; \frac{b(1)}{a(1)} \;=\; Q(1) \;=\; 1 \;>\; 0 ,
\qquad\text{so}\qquad \xi(k) = 1 + k + O(k^2) .
\end{equation}
Equivalently $z(k) = 1 - k + O(k^2)$: the dc pole leaves $z = 1$ along the negative real direction and enters the disc.
Hence there is $k_0 > 0$ such that the mode is asymptotically stable for every $k \in (0, k_0)$, and the locus starts inside the disc.

\medskip\noindent\textit{Step 2: the crossings are exactly $\Omega_+(Q)$.}
For $k > 0$, its corresponding characteristic polynomial $c_k$ has a root on $|\xi| = 1$ if and only if~(\ref{eq:appx_char_z}) is solved by some $z = e^{i\omega}$ on the unit circle, i.e.
\begin{equation}
\label{eq:appx_crossing_gain}
k \;=\; \frac{1 - e^{i\omega}}{Q(e^{i\omega})} \;=:\; k(\omega) .
\end{equation}
The frequency $\omega = 0$ forces $k = 0$ and is excluded.
Because $a, b$ are real, $k(-\omega) = \overline{k(\omega)}$, so the crossings appear symmetrically, i.e., $\pm\omega$ occur at the same gain, and it suffices to take $\omega \in (0, \pi]$.
Finally $k$ must be a real positive number, which is precisely the defining condition of $\Omega_+(Q)$ in~(\ref{eq:gain_at_first_crossing}).
On $\Omega_+(Q)$ the value $k(\omega)$ coincides with its modulus, i.e., $k(\omega) = |k(\omega)|$.

\medskip\noindent\textit{Step 3: conclusion.}
Combining Steps 1 and 2, the set of gains at which some root sits on the unit circle is $\{ k(\omega) : \omega \in \Omega_+(Q) \}$, and the mode is stable on the whole interval below its smallest element.
Therefore
\begin{equation}
\label{eq:appx_kstar}
k_\star(Q) \;=\; \min_{\omega \in \Omega_+(Q)} \left| \frac{1 - e^{i\omega}}{Q(e^{i\omega})} \right| ,
\end{equation}
with the convention $k_\star(Q) = +\infty$ and $\Gamma = 0$ when $\Omega_+(Q) = \emptyset$, which is exactly the claim~(\ref{eq:gain_at_first_crossing}).
Writing $\Gamma := 2 / k_\star(Q)$, the mode is stable for all sufficiently small positive $\eta \lambda < 2/\Gamma$, and marginal at equality, i.e.\ iff $\eta \lambda \Gamma \le 2$, which recovers~(\ref{eq:marginal_stability_condition}).
Since $\Gamma$ depends on $Q$ alone and not on $\lambda$ (mode decoupling), the binding mode is the sharpest one, and the system is not exponentially unstable iff $\eta \lambda_{\max}(\mP^{-1/2}\mH\mP^{-1/2}) \, \Gamma \le 2$, which is precisely the boxed condition~(\ref{eq:marginal_stability_condition}).

\medskip\noindent\textit{Step 4: the Nyquist case.}
If the minimum in~(\ref{eq:appx_kstar}) is attained at the Nyquist frequency $\omega = \pi$, i.e., $z = e^{i\omega} = -1$, and $Q(-1) > 0$, then
\begin{equation}
\label{eq:appx_nyquist_gauge}
k_\star(Q) \;=\; \frac{1 - (-1)}{Q(-1)} \;=\; \frac{2}{Q(-1)} ,
\qquad\text{hence}\qquad
\Gamma \;=\; \frac{2}{k_\star(Q)} \;=\; Q(-1) ,
\end{equation}
and the stability condition reads $\eta \lambda Q(-1) \le 2$.
The marginal root $z = -1$ has $z^2 = 1$, so the marginal motion is a period-2 oscillation.
This proves the last sentence of Proposition~\ref{prop:edge_preconditioner}.
\end{proof}

\subsection{Two Lemmas on the Root Locus Gauge}
\label{appx:proof:gauge_lemmas}

The gauge $\Gamma$ is a functional of the filter $Q$ alone, so it can be computed once per optimizer family, and is not dependent on the preconditioner.
The following two lemmas consolidates this: the first removes the dc normalization $Q(1)$, and the second identifies a large class of filters for which the escape is provably at Nyquist $z = -1$.

\begin{lemma}[Positive scaling]
\label{lem:appx_scaling}
Let $c > 0$. Then $\Omega_+(cQ) = \Omega_+(Q)$, $k_\star(cQ) = k_\star(Q)/c$, and
\begin{equation}
\label{eq:appx_scaling}
\Gamma(cQ) \;=\; c \, \Gamma(Q) .
\end{equation}
In particular the dc-normalized filter $\tilde Q := Q / Q(1)$ satisfies $\Gamma(\tilde Q) = \Gamma(Q) / Q(1)$, and $\Gamma(Q) = Q(-1)$ holds for $Q$ if and only if it holds for $\tilde Q$.
\end{lemma}

\begin{proof}
$k(\omega)$ in~(\ref{eq:appx_crossing_gain}) is replaced by $k(\omega)/c$.
Multiplication by $c^{-1} > 0$ preserves membership in $\R_{>0}$, so $\Omega_+$ is unchanged, and it scales every modulus by $c^{-1}$, so the minimum scales by $c^{-1}$ and $\Gamma = 2/k_\star$ by $c$.
The last claim follows since $\Gamma$ and $Q(-1)$ carry the same factor $c$.
\end{proof}

\begin{lemma}[Nyquist escape for positive one-pole mixtures]
\label{lem:appx_nyquist_mixture}
Let
\begin{equation}
\label{eq:appx_mixture}
Q(z) \;=\; \sum_{j=1}^{J} w_j \, \frac{1 - \beta_j}{1 - \beta_j z^{-1}} ,
\qquad w_j > 0, \quad \sum_{j=1}^{J} w_j = 1, \quad \beta_j \in [0, 1) .
\end{equation}
Then $Q(1) = 1$, hypotheses (H1) and (H2) hold, and
\begin{equation}
\label{eq:appx_mixture_gauge}
\Omega_+(Q) = \{\pi\},
\qquad
k_\star(Q) = \frac{2}{Q(-1)},
\qquad
\Gamma \;=\; Q(-1) \;=\; \sum_{j=1}^{J} w_j \, \frac{1 - \beta_j}{1 + \beta_j} \;\in\; (0, 1] ,
\end{equation}
with $\Gamma = 1$ if and only if every $\beta_j = 0$.
\end{lemma}

\begin{proof}
Each summand equals $w_j$ at $z = 1$, so $Q(1) = 1$, giving (H2).
Also, the poles are $z = \beta_j \in [0,1)$, giving (H1).
For $\omega \in (0, \pi]$ the elementary identity
\begin{equation}
\label{eq:appx_half_angle}
1 - e^{i\omega} \;=\; 2 \sin\left(\frac{\omega}{2}\right) \; e^{i(\omega - \pi)/2},
\qquad 2\sin\left(\frac{\omega}{2}\right) > 0 ,
\end{equation}
shows that $k(\omega) \in \R_{>0}$ holds iff $e^{i\theta} / Q(e^{i\omega}) \in \R_{>0}$ with $\theta := (\omega - \pi)/2$.
Writing $e^{i\theta}/Q = e^{i\theta}\overline{Q}/|Q|^2$ and take the conjugate.
The reality requirement $e^{i\theta} / Q(e^{i\omega}) \in \R_{>0}$ is exactly the vanishing of the \emph{phase margin}
\begin{equation}
\label{eq:appx_phase_margin}
\Psi(\omega) \;:=\; \operatorname{Im}\!\left( e^{-i(\omega - \pi)/2} \, Q(e^{i\omega}) \right) .
\end{equation}
So it suffices to show that $\Psi > 0$ on $(0,\pi)$, which rules out every interior crossing.

Consider one pole, $Q_\beta(z) := (1-\beta)/(1 - \beta z^{-1})$, so that $Q_\beta(e^{i\omega}) = (1-\beta)/(1 - \beta e^{-i\omega})$.
Rationalizing gives
\begin{equation}
\label{eq:appx_single_pole_margin}
\Psi_\beta(\omega)
\;=\; \frac{1 - \beta}{|1 - \beta e^{-i\omega}|^2} \,
\operatorname{Im}\!\left( e^{-i\theta} \left( 1 - \beta e^{i\omega} \right) \right)
\;=\; \frac{1 - \beta}{|1 - \beta e^{-i\omega}|^2}
\left( -\sin\theta - \beta \sin(\omega - \theta) \right) .
\end{equation}
Since $-\theta = \tfrac{\pi - \omega}{2}$ and $\omega - \theta = \tfrac{\omega + \pi}{2}$, both sines collapse onto the same value, $-\sin\theta = \cos\tfrac{\omega}{2} = \sin(\omega - \theta)$, and we have
\begin{equation}
\label{eq:appx_single_pole_margin_closed}
\Psi_\beta(\omega)
\;=\; \frac{(1-\beta)^2 \, \cos(\omega/2)}{1 - 2\beta\cos\omega + \beta^2}
\;>\; 0
\qquad \text{for } \omega \in (0, \pi),
\end{equation}
because $\beta < 1$ and $\cos(\omega/2) > 0$ for $\omega \in (0, \pi)$.
Also, the denominator $|1 - \beta e^{-i\omega}|^2 \ge (1-\beta)^2 > 0$.
The single pole therefore clears the critical line strictly, and it does so by a margin proportional to $(1-\beta)^2$: the closer $\beta$ is to $1$, the closer the filter comes to admitting an interior crossing.

Since $\Psi$ is $\R$-linear in $Q$, the mixture inherits the sign:
$\Psi(\omega) = \sum_j w_j \Psi_{\beta_j}(\omega) > 0$ on $(0,\pi)$, as every $w_j > 0$.
Hence $\Omega_+(Q) \cap (0, \pi) = \emptyset$.
At $\omega = \pi$ we have $\theta = 0$ and $Q(-1) = \sum_j w_j (1-\beta_j)/(1+\beta_j) > 0$ is real, so $\Psi(\pi) = 0$ and $k(\pi) = 2/Q(-1) \in \R_{>0}$, i.e.\ $\pi \in \Omega_+(Q)$.
Thus $\Omega_+(Q) = \{\pi\}$ and~(\ref{eq:appx_mixture_gauge}) follows from
Step 4 of Appendix~\ref{appx:proof:eos_prop}.
Finally $(1-\beta)/(1+\beta) \in (0, 1]$ with equality iff $\beta = 0$, and a convex combination of such numbers obeys the same bounds.
\end{proof}

\begin{remark}[Root locus on various filters]
\label{rem:appx_positivity}
The vast majority of existing gradient-based optimizers rely on one-pole momentum, either in Heavy Ball~\citep{polyak1964some} or Nesterov's Accelerated Gradient~\citep{nesterov1983method}.
As the following section shows, these are the main examples where Lemma~\ref{lem:appx_nyquist_mixture} applies.
The actual root loci of various optimizers with and without preconditioners are shown in Figure~\ref{fig:rootlocus} of the main manuscript and in subsequent figures in Appendix~\ref{sec:more_math}.
It is clear that (1) the root locus of optimizers using a single momentum parameter $\beta$ (or $\beta_1$) is identical and independent of their preconditioners, and (2) their crossings always appear at the Nyquist frequency.
\end{remark}

\subsection{Proof of Corollary~\ref{cor:cohen_separable_momentum_coefficient}}
\label{appx:proof:cohen_gamma}

All that remains is to write each of the widely-used optimizer filters in the form~(\ref{eq:appx_mixture}) and read off $Q(-1)$.
We use the convention that $\vu_t$ is the vector multiplying $\eta$ in~(\ref{eq:gradient_filter}) with $\mP = \mI$, and $\vm_t$ denotes the momentum buffer.

\begin{proof}
\medskip\noindent\textit{Heavy Ball.}
From $\vm_t = \beta \vm_{t-1} + \vg_t$ and $\vu_t = \vm_t$ we get $\vm_t = \sum_{k \ge 0}\beta^k \vg_{t-k}$, the geometric filter quoted below~(\ref{eq:gradient_filter}), so
\begin{equation}
\label{eq:appx_Q_hb}
Q^{\mathrm{HB}}(z) \;=\; \sum_{k \ge 0} \beta^k z^{-k} \;=\; \frac{1}{1 - \beta z^{-1}} ,
\qquad Q^{\mathrm{HB}}(1) = \frac{1}{1 - \beta} .
\end{equation}
Its dc-normalized version $Q^{\mathrm{nHB}} = (1-\beta) Q^{\mathrm{HB}}$ is exactly~(\ref{eq:appx_mixture}) with $J = 1$, $w_1 = 1$, $\beta_1 = \beta$, so Lemma~\ref{lem:appx_nyquist_mixture} applies and gives $\Gamma^{\mathrm{nHB}} = (1-\beta)/(1+\beta)$.
To recover unnormalized results, we apply Lemma~\ref{lem:appx_scaling} with $c = 1/(1-\beta)$ yielding $\Gamma^{\mathrm{HB}} = 1/(1+\beta)$.
Both equal the respective $Q(-1)$ and recover the results of~\citet{cohen2023adaptive}.

\medskip\noindent\textit{Nesterov's momentum.}
With $\vm_t = \beta \vm_{t-1} + \vg_t$ and $\vu_t = \vg_t + \beta \vm_t$,
\begin{equation}
\label{eq:appx_Q_nag}
Q^{\mathrm{NAG}}(z) \;=\; 1 + \frac{\beta}{1 - \beta z^{-1}} ,
\qquad Q^{\mathrm{NAG}}(1) = \frac{1}{1-\beta} .
\end{equation}
Multiplying by $(1-\beta)$ and splitting the constant term off as a pole at the origin,
\begin{equation}
\label{eq:appx_Q_nnag_mixture}
Q^{\mathrm{nNAG}}(z)
\;=\; (1-\beta) \cdot 1 \;+\; \beta \cdot \frac{1 - \beta}{1 - \beta z^{-1}} ,
\end{equation}
which is~(\ref{eq:appx_mixture}) with weights $(1-\beta, \beta)$ and poles $(0, \beta)$.
The weights are positive and sum to one for $\beta \in (0,1)$, so Lemma~\ref{lem:appx_nyquist_mixture} applies and
\begin{equation}
\label{eq:appx_gamma_nnag}
\Gamma^{\mathrm{nNAG}}
\;=\; (1-\beta) \cdot 1 + \beta \cdot \frac{1-\beta}{1+\beta}
\;=\; \frac{(1-\beta)(1 + 2\beta)}{1 + \beta} ,
\qquad
\Gamma^{\mathrm{NAG}} \;=\; \frac{1 + 2\beta}{1 + \beta} ,
\end{equation}
the second by Lemma~\ref{lem:appx_scaling}.
Again, both recover the results reported in~\citet{cohen2023adaptive}.

This proves Corollary~\ref{cor:cohen_separable_momentum_coefficient}: for both Heavy Ball and
Nesterov's momentum, normalized or not, $\Omega_+ = \{\pi\}$ and therefore
$\Gamma(\beta) = Q(-1)$, matching the values quoted below~(\ref{eq:edge_preconditioner}).
\end{proof}

Similar results on the remaining filters of Section~\ref{sec:observation}, such as QHM~\citep{ma2018quasi}, parallel dual momentum, and cascaded two-pole momenta~\citep{lee2024grokfast,pagliardini2025ademamix} are derived in Appendix~\ref{appx:proof:gamma_table} with their gauges summarized in Table~\ref{tab:appx_gauge_table}.
Also refer to the root locus plots in Appendix~\ref{sec:more_math} for visualization of these results.

\subsection{Proof of Corollary~\ref{cor:coincidence}}
\label{appx:proof:coincidence}

This subsection proves Corollary~\ref{cor:coincidence} and the GD reduction used in Theorem~\ref{thm:nonseparability}.
Fix an iteration and recall the definitions $A = \vg^\top \vu$, $B = \eta \vu^\top \bar{\mH}_{\eta,\vu} \vu / 2$, $\|\vu\|_{\mP}^2 = \vu^\top \mP \vu$,
\begin{equation}
\label{eq:appx_coincidence_defs}
\zeta \;=\; \frac{B}{A}, \qquad
\chi \;=\; \frac{\eta S \Gamma}{2}, \qquad
\varsigma \;=\; \frac{D^{\|\cdot\|_{\mP}}}{S} \;=\; \frac{\vu^\top \bar{\mH}_{\eta,\vu} \vu}{S \|\vu\|_{\mP}^2}, \qquad
\tau \;=\; \frac{\Gamma A}{\|\vu\|_{\mP}^2}.
\end{equation}
The assumptions $\eta, S, \Gamma > 0$, $\vu \neq \vzero$, and $A \neq 0$ ensure all four quantities are finite, with $\chi > 0$ and $\tau \neq 0$.
This identity provides additional insight into when the worst-case circumscription $\chi$ is fully realized along the actual optimization trajectory, which is summarized in the following proposition.

\begin{proof}
\medskip\noindent\textit{The decomposition is exact.}
Substituting~(\ref{eq:appx_coincidence_defs}) directly,
\begin{equation}
\label{eq:appx_exact_decomposition}
\chi \frac{\varsigma}{\tau}
\;=\; \frac{\eta S \Gamma}{2} \cdot \frac{\vu^\top \bar{\mH}_{\eta,\vu} \vu}{S \|\vu\|_{\mP}^2} \cdot \frac{\|\vu\|_{\mP}^2}{\Gamma A}
\;=\; \frac{\eta}{2} \frac{\vu^\top \bar{\mH}_{\eta,\vu} \vu}{\vg^\top \vu}
\;=\; \frac{B}{A} \;=\; \zeta.
\end{equation}
Every occurrence of $S$, $\Gamma$, and $\|\vu\|_{\mP}^2$ cancels, so~(\ref{eq:optimizer_aware_eos_factor}) is an identity: neither $\bar{\mH}_{\eta,\vu} \simeq \mH$ nor occupation of $\vv_{\max}$ is used anywhere in~(\ref{eq:appx_exact_decomposition}).
This is precisely why the decomposition constrains only the ratio $\varsigma / \tau$.

\medskip\noindent\textit{Coincidence condition.}
From~(\ref{eq:appx_exact_decomposition}) and $\chi > 0$, it is trivial that $\zeta = \chi$ if and only if $\varsigma = \tau$.
Clearing the common denominator $\|\vu\|_{\mP}^2 > 0$ and multiplying by $S > 0$,
\begin{equation}
\label{eq:appx_coincidence_primitive}
\zeta = \chi
\iff \varsigma = \tau
\iff \frac{\vu^\top \bar{\mH}_{\eta,\vu} \vu}{S} = \Gamma A
\iff \vu^\top \bar{\mH}_{\eta,\vu} \vu = S \Gamma \, \vg^\top \vu,
\end{equation}
which gives the coincidence conditio of (\ref{eq:appx_exact_decomposition}).

\medskip\noindent\textit{Conditions on the individual factors.}
Because $S > 0$, the definition of $\varsigma$ gives $\varsigma = 1 \iff \vu^\top \bar{\mH}_{\eta,\vu} \vu / \|\vu\|_{\mP}^2 = S$.
Because $\|\vu\|_{\mP}^2 > 0$, the definition of $\tau$ gives $\tau = 1 \iff \Gamma A = \vu^\top \mP \vu \iff \vu^\top (\Gamma \vg - \mP \vu) = 0$.
Sufficiency of $\varsigma = \tau = 1$ is immediate from~(\ref{eq:appx_exact_decomposition}).
Necessity fails, and Appendix~\ref{appx:remark:coincidence_counterexample} exhibits an admissible instance with $\varsigma = \tau \neq 1$.

\medskip\noindent\textit{Form at the realized edge.}
Setting $\zeta = 1$ in~(\ref{eq:appx_exact_decomposition}) and solving for $\chi$ gives $\chi = \tau / \varsigma$.
If $\bar{\mH}_{\eta,\vu} \simeq \mH$, then $D^{\|\cdot\|_{\mP}} \le S$ by the variational characterization of $S$ recalled in Section~\ref{sec:theory}, hence $\varsigma \le 1$ and $\chi \ge \tau$.

\medskip\noindent\textit{GD reduction.}
Take $\mP = I$ and $Q \equiv 1$. Then $\Gamma = Q(-1) = 1$ and $\vu = \vg$, so
$A = \vg^\top \vu = \|\vg\|^2$ and $\|\vu\|_{\mP}^2 = \|\vg\|^2$, hence $\tau = \Gamma A / \|\vu\|_{\mP}^2 = 1$.
The identity collapses to $\zeta = \chi \varsigma$ with $\chi = \eta S / 2$.
On a frozen quadratic, $\bar{\mH}_{\eta,\vu} = \mH$ and $\varsigma = \vg^\top \mH \vg / (S \|\vg\|^2)$, which equals $1$ if and only if $\vg$ is an eigen-direction for $\lambda_{\max}(\mH)$.
This is the occupation hypothesis of Section~\ref{sec:preliminary}, and is the content of the GD clause in Corollary~\ref{cor:coincidence}.
\end{proof}

\subsection{Proof of Corollary~\ref{cor:mode_locked_coincidence}}
\label{appx:proof:mode_locked}

\begin{proof}
\medskip\noindent\textit{Notations.}
Write $\vv = \vv_{\max}$ for the unit top eigenvector of $\mP^{-1/2} \mH \mP^{-1/2}$, so that
\begin{equation}
\label{eq:appx_ml_eigen}
\mP^{-1/2} \mH \mP^{-1/2} \vv \;=\; S \vv
\qquad\Longleftrightarrow\qquad
\mH \mP^{-1/2} \vv \;=\; S \mP^{1/2} \vv .
\end{equation}
This will be used throughout the proof.

\medskip\noindent\textit{The locked mode is invariant.}
By (Q1) the loss is quadratic over the horizon, so $\vg_s = \mH (\vtheta_s - \vtheta_\star)$.
With (Q3), we have $x_s \mH \mP^{-1/2} \vv = S x_s \mP^{1/2} \vv$ using~(\ref{eq:appx_ml_eigen}).
By (Q2) the preconditioner may be pulled out of the convolution, so the update~(\ref{eq:gradient_filter}) is
\begin{equation}
\label{eq:appx_ml_update}
\vu_t \;=\; \mP^{-1} (Q * \vg_{\le t})_t \;=\; S (Q * x_{\le t})_t \, \mP^{-1} \mP^{1/2} \vv \;=\; S (Q * x_{\le t})_t \, \mP^{-1/2} \vv .
\end{equation}
The update therefore stays in $\operatorname{span}(\mP^{-1/2}\vv)$, which is the locked mode.

\medskip\noindent\textit{Spatial factor.}
Abbreviate $q := (Q * x_{\le t})_t$ and use $\vv^\top \vv = 1$ together with~(\ref{eq:appx_ml_eigen}) to get:
\begin{equation}
\label{eq:appx_ml_quadratics}
\begin{aligned}
A \;&=\; \vg_t^\top \vu_t \;=\; S^2 x_t q, \\
\|\vu_t\|_{\mP}^2 \;&=\; \vu_t^\top \mP \vu_t \;=\; S^2 q^2,\\
\vu_t^\top \mH \vu_t \;&=\; S^2 q^2 \, \vv^\top \mP^{-1/2} \mH \mP^{-1/2} \vv \;=\; S^3 q^2 .
\end{aligned}
\end{equation}
Since $\bar{\mH}_{\eta,\vu} = \mH$ by (Q1), the spatial participation factor is
\begin{equation}
\label{eq:appx_ml_varsigma}
\varsigma \;=\; \frac{\vu_t^\top \bar{\mH}_{\eta,\vu} \vu_t}{S \|\vu_t\|_{\mP}^2} \;=\; \frac{S^3 q^2}{S \cdot S^2 q^2} \;=\; 1 .
\end{equation}
Therefore, the conditions (Q1)--(Q3) are sufficient for the spatial factor to be $\varsigma = 1$.
The temporal factor is $\tau = \Gamma A / \|\vu_t\|_{\mP}^2 = \Gamma S^2 x_t q / (S^2 q^2) = \Gamma x_t / q$, which recovers~(\ref{eq:mode_locked_tau}).

\medskip\noindent\textit{Temporal factor.}
By (Q4) the escape is at $z = -1$, so $\Gamma = Q(-1)$ by Proposition~\ref{prop:edge_preconditioner}, and the period-2 history $x_{t-k} = (-1)^k x_t$ gives
\begin{equation}
\label{eq:appx_ml_q}
q \;=\; \sum_{k \ge 0} Q_k x_{t-k} \;=\; \Big( \sum_{k \ge 0} Q_k (-1)^k \Big) x_t \;=\; Q(-1) \, x_t \;=\; \Gamma x_t .
\end{equation}
Hence $\tau = \Gamma x_t / q = 1$, and $A = S^2 x_t q = \Gamma S^2 x_t^2 > 0$ because $\Gamma > 0$ and $x_t \neq 0$, so $A \neq 0$ and the identity~(\ref{eq:appx_exact_decomposition}) applies.
Substituting $\varsigma = \tau = 1$ into the identity~(\ref{eq:appx_exact_decomposition}) gives $\zeta = \chi = \eta S \Gamma / 2 = \eta S Q(-1)/2$.
Note that~(\ref{eq:appx_ml_q}) is what requires the Nyquist frequency: a mode oscillating at any other escape frequency $\omega_\star$ contributes a phase shift, so no real gain satisfies $q = \Gamma x_t$ at every step.
\end{proof}

\subsection{A Counterexample to the Necessity of \texorpdfstring{$\varsigma = \tau = 1$}{sigma = tau = 1}}
\label{appx:remark:coincidence_counterexample}

The following instance lies entirely inside the framework of~(\ref{eq:gradient_filter}) and is exactly computable.
Take the quadratic objective $\gL(\vtheta) = \tfrac{1}{2} \vtheta^\top \mH \vtheta$ on $\R^2$ with $\mH = \operatorname{diag}(4, 1)$, so that $\bar{\mH}_{\eta,\vu} = \mH$ \emph{exactly} and $\vv_{\max} = \ve_1$.
Take $\mP = \mI$, so $S = 4$, and take the dc-normalized Heavy Ball filter $Q_k = (1 - \beta) \beta^k$ with $\beta = 1/2$, so $Q(1) = 1$ and $\Gamma = Q(-1) = (1 - \beta)/(1 + \beta) = 1/3$.
Let $\eta = 1/10$ and start from a zero momentum buffer, so that $\vu_0 = (1 - \beta) \vg_0$ and
\begin{equation}
\label{eq:appx_ce_update}
\vg_1 \;=\; (\mI - \eta (1 - \beta) \mH) \vg_0 \;=\; \operatorname{diag}(\tfrac{4}{5}, \tfrac{19}{20}) \vg_0,
\qquad
\vu_1 \;=\; \operatorname{diag}(\tfrac{13}{20}, \tfrac{29}{40}) \vg_0 .
\end{equation}
All four quantities are invariant to the scale of $\vg_0$, so only its direction matters.
Reading them at $t = 1$:
\begin{equation}
\label{eq:appx_ce_endpoints}
\begin{aligned}
\vg_0 = \ve_1 : \quad &\varsigma = 1, &\tau &= \tfrac{16}{39}, &\chi &= \tfrac{1}{15}, &\zeta &= \tfrac{13}{80}; \\
\vg_0 = \ve_2 : \quad &\varsigma = \tfrac{1}{4}, &\tau &= \tfrac{38}{87}, &\chi &= \tfrac{1}{15}, &\zeta &= \tfrac{29}{760}.
\end{aligned}
\end{equation}
Two conclusions can made from this.
First, the case $\vg_0 = \ve_1$ satisfies every geometric condition one would attach to coincidence, namely $\bar{\mH}_{\eta,\vu} = \mH$ exactly and $\vu_1 \in \operatorname{span}(\vv_{\max})$, hence $\varsigma = 1$; yet $\zeta / \chi = 39/16 \neq 1$, because $\Gamma$ is the loop gain at the root-locus escape frequency while the trajectory at $t = 1$ is not at that frequency.
Second, writing $\vg_0(\phi) = (\cos\phi, \sin\phi)$, the map $\phi \mapsto \varsigma - \tau$ is continuous on $[0, \pi/2]$ with $\varsigma - \tau = 23/39 > 0$ at $\phi = 0$ and $\varsigma - \tau = -65/348 < 0$ at $\phi = \pi/2$, so it has a root $\phi_\star \in (0, \pi/2)$.
At $\phi_\star$ we have $\varsigma = \tau$ and therefore $\zeta = \chi$, while $\varsigma < 1$ because $\varsigma = 1$ forces $\vu_1 \in \operatorname{span}(\ve_1)$, i.e. $\phi = 0$.
Numerically $\phi_\star = 1.01024954$ and $\varsigma = \tau = 0.43040137$, giving $\zeta = \chi = 1/15$ to ten significant digits.
Hence $\varsigma = \tau = 1$ is not necessary for $\zeta = \chi$.

\subsection{Proof of Theorem~\ref{thm:nonseparability}}
\label{appx:proof:nonseparability}

\begin{proof}
\medskip\noindent\textit{From the identity.}
Corollary~\ref{cor:coincidence} gives $\zeta = \chi \varsigma / \tau$ and $\zeta = \chi$ if and only if $\varsigma = \tau$.
The scalar $\chi = \eta S \Gamma / 2$ is a function of $(\eta, S, \Gamma)$ alone.
Therefore $\zeta$ is a function of $(\eta, S, \Gamma)$ if and only if the ratio $\varsigma / \tau$ is.

\medskip\noindent\textit{GD reduction.}
For $\mP = I$ and $Q \equiv 1$, the GD clause of Corollary~\ref{cor:coincidence} gives $\Gamma = 1$, $\tau \equiv 1$, and $\zeta = \chi \varsigma$ with $\chi = \eta S / 2$.
Hence $\chi$ determines $\zeta$ if and only if occupation $\varsigma$ is a function of $(\eta, S)$ alone.

\medskip\noindent\textit{Occupation is not a function of $S$.}
On a frozen quadratic, $\bar{\mH}_{\eta,\vu} = \mH$ and $\vu = \vg$, so
$\varsigma = \vg^\top \mH \vg / (S \|\vg\|^2)$.
This equals $1$ if and only if $\vg$ lies in the top eigenspace of $\mH$.
In dimension at least $2$ that is not implied by $S = \lambda_{\max}(\mH)$.
The instance $\mH = \operatorname{diag}(4,1)$ has $S = 4$ independently of $\vg$, but $\varsigma = 1$ along $\ve_1$ and $\varsigma = 1/4$ along $\ve_2$.
The same $(\eta, S, \Gamma = 1)$ therefore yields two values of $\zeta$.
\end{proof}

\section{Gauge Derivations}
\label{sec:more_math}
Proposition~\ref{prop:edge_preconditioner} assigns a scalar gauge $\Gamma := 2 / k_\star(Q)$ to every causal filter $Q$.
It is the loop gain at the first positive-gain unit-circle crossing of the root locus, abbreviated as the \emph{escape gain} in the main text.
Corollary~\ref{cor:cohen_separable_momentum_coefficient} is the special case $\Gamma = Q(-1)$, which holds for Heavy Ball and Nesterov's momentum and is proved in Appendix~\ref{appx:proof:cohen_gamma} from Lemmas~\ref{lem:appx_scaling} and~\ref{lem:appx_nyquist_mixture}.
This additional section extends this to include $Q(z)$ and $\Gamma$ for every optimizer that appears in the main text, especially those with $\Gamma \neq Q(-1)$, which are the filters that violate hypothesis (Q4) of Corollary~\ref{cor:mode_locked_coincidence} on the filter side.

\begin{table}[t]
\centering
\small
\setlength{\tabcolsep}{4.5pt}
\caption{
\small
Root-locus gauge of the \emph{one-pole and parallel} optimizer filters.
Every dc-normalized entry is a positive one-pole mixture~(\ref{eq:appx_mixture}), so Lemma~\ref{lem:appx_nyquist_mixture} certifies $\Omega_+ = \{\pi\}$ and $\Gamma = Q(-1)$.
Cascaded two-pole momenta are \emph{not} in this table.
The last column recovers the results from \citet{cohen2023adaptive}.
}
\label{tab:appx_gauge_table}
\begin{tabular}{@{}lcccc@{}}
\toprule
Optimizer & $Q(z)$ & $Q(1)$ & mixture $(w_j; \beta_j)$ & $\Gamma = Q(-1)$ \\
\midrule
GD
& $1$
& $1$
& $(1; 0)$
& $1$ \\[2pt]
Heavy Ball
& $\dfrac{1}{1 - \beta z^{-1}}$
& $\dfrac{1}{1-\beta}$
& ---
& $\dfrac{1}{1+\beta}$ \\[8pt]
normalized HB
& $\dfrac{1-\beta}{1 - \beta z^{-1}}$
& $1$
& $(1; \beta)$
& $\dfrac{1-\beta}{1+\beta}$ \\[8pt]
Nesterov
& $1 + \dfrac{\beta}{1 - \beta z^{-1}}$
& $\dfrac{1}{1-\beta}$
& ---
& $\dfrac{1+2\beta}{1+\beta}$ \\[8pt]
normalized Nesterov
& $(1-\beta) + \dfrac{\beta(1-\beta)}{1 - \beta z^{-1}}$
& $1$
& $(1-\beta, \beta; 0, \beta)$
& $\dfrac{(1-\beta)(1+2\beta)}{1+\beta}$ \\[8pt]
QHM, $\nu \in [0,1]$
& $(1-\nu) + \dfrac{\nu(1-\beta)}{1 - \beta z^{-1}}$
& $1$
& $(1-\nu, \nu; 0, \beta)$
& $\dfrac{1 + \beta(1-2\nu)}{1+\beta}$ \\[8pt]
parallel dual momentum, $\alpha > 0$
& $\displaystyle \sum_{j=1,2} \frac{w_j (1-\beta_j)}{1 - \beta_j z^{-1}}$
& $1$
& $\dfrac{(1, \alpha)}{1+\alpha}; (\beta_1, \beta_2)$
& $\displaystyle \sum_{j=1,2} w_j \frac{1-\beta_j}{1+\beta_j}$ \\
\bottomrule
\end{tabular}
\end{table}

\subsection{QHM and Parallel Dual Momentum}
\label{appx:proof:gamma_table}

We use the convention that $\vu_t$ is the vector multiplying $\eta$ in~(\ref{eq:gradient_filter}) with $\mP = \mI$, and $\vm_t$ denotes the momentum buffer.
Every dc-normalized filter in this subsection is a positive one-pole mixture~(\ref{eq:appx_mixture}), so Lemma~\ref{lem:appx_nyquist_mixture} applies and gives $\Omega_+ = \{\pi\}$ and $\Gamma = Q(-1)$.

\paragraph{QHM.}
QHM~\citep{ma2018quasi} uses the normalized buffer $\vm_t = \beta\vm_{t-1} + (1-\beta)\vg_t$ and $\vu_t = (1-\nu)\vg_t + \nu \vm_t$, which is~(\ref{eq:appx_mixture}) with weights $(1-\nu, \nu)$ and poles $(0, \beta)$; it is already dc-normalized, i.e., $Q(1) = 1$.
For $\nu \in (0,1)$ the weights are strictly positive.
Lemma~\ref{lem:appx_nyquist_mixture} gives
\begin{equation}
\label{eq:appx_gamma_qhm}
\Gamma^{\mathrm{QHM}}(\beta, \nu) \;=\; \frac{1 + \beta(1 - 2\nu)}{1 + \beta} .
\end{equation}
The endpoints $\nu = 0$ and $\nu = 1$ are one-pole reductions (GD and normalized Heavy Ball), both still with $\Gamma = Q(-1)$.
Two special cases are immediate from~(\ref{eq:appx_gamma_qhm}) and~(\ref{eq:appx_Q_nnag_mixture}):
$\nu = 1$ recovers normalized Heavy Ball, and $\nu = \beta$ recovers \emph{normalized Nesterov}, since $1 + \beta - 2\beta^2 = (1-\beta)(1+2\beta)$.
The experimental QHM arm of Section~\ref{sec:observation} uses $\nu = 0.7 \in (0,1)$, so it remains in this Nyquist class with $\Gamma = Q(-1)$.

\begin{figure}[t]
\centering
\includegraphics[width=0.7\linewidth]{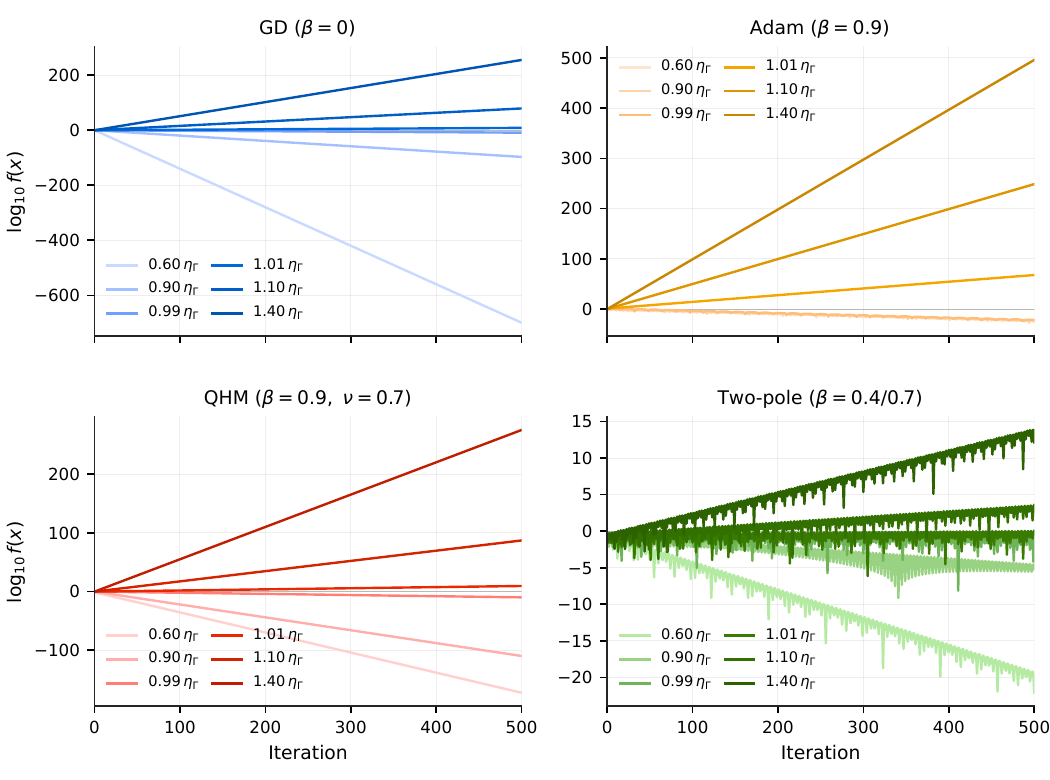}
\vspace{-1em}
\caption{\small
We optimize a quadratic loss $f(x) = \tfrac12 x^2$ with four types of optimizers: GD, Adam, QHM, and cascade two-pole at various learning rates $\eta$.
Observe that the stability bound of $\chi = \eta S \Gamma / 2 \le 1$ is always satisfied for this quadratic loss.
}
\label{fig:gaugetest}
\end{figure}

\begin{table}[t]
\centering
\small
\setlength{\tabcolsep}{3.5pt}
\caption{\small 
Predicted and observed quadratic thresholds.
$\Gamma_{\mathrm{exp}}$ is the experiment-report gauge.
The observed $\eta$ is the last stable and first unstable $\eta$ on a
$0.005$-relative grid.
$605$ sweep runs in total.}
\label{tab:appx_gamma_thresholds}
\begin{tabular}{@{}lccccc@{}}
\toprule
Filter & $\Gamma_{\mathrm{exp}}$ / 2 & $2/\Gamma_{\mathrm{exp}}$ & $2/Q(-1)$ & $\omega_\star/\pi$ & observed $\eta$ \\
\midrule
GD
  & $0.5$ & $2$ & $2$ & $1.000$ & $[2,\,2.01]$ \\
EMA-HB $\beta=0.9$
  & $0.02632$ & $38$ & $38$ & $1.000$ & $[37.81,\,38]$ \\
QHM $\beta=0.9,\nu=0.7$
  & $0.1684$ & $5.9375$ & $5.9375$ & $1.000$ & $[5.908,\,5.938]$ \\
two-pole $0.4/0.7$
  & $0.3889$ & $2.571$ & $26.44$ & $0.263$ & $[2.571,\,2.584]$ \\
two-pole $0.7/0.85$
  & $1.469$ & $0.681$ & $139.8$ & $0.088$ & $[0.681,\,0.684]$ \\
\bottomrule
\end{tabular}
\end{table}

\paragraph{Parallel Dual Momentum.}
Dual momentum in its \emph{parallel} realization, two buffers at $\beta_1 \neq \beta_2$ summed as $\vu_t = (\vm^{(1)}_t + \alpha \vm^{(2)}_t)/(1+\alpha)$, is again~(\ref{eq:appx_mixture}), now with $J = 2$, weights $(1, \alpha)/(1+\alpha)$ and poles $(\beta_1, \beta_2)$, so for $\alpha > 0$
Lemma~\ref{lem:appx_nyquist_mixture} applies verbatim, with no restriction on the number of poles.
Table~\ref{tab:appx_gauge_table} collects these results.
In our text, we did not use this construction for two-pole momenta (see the next section), but we included them in the table for completeness.

\subsection{Two-Pole Momentum: the Gauge is No Longer $Q(-1)$}
\label{appx:proof:two_pole}

Everything so far has been a \emph{parallel} construction: gradients are filtered by several one-pole blocks and the outputs are added.
The two-pole filters~\citep{lee2024grokfast} of Section~\ref{sec:observation} are instead built by \emph{cascading} a momentum buffer with a second one-pole block.
Define the two-pole momentum family, dc-normalized, with $\beta_1, \beta_2 \in (0,1)$ and $\alpha \ge 0$,
\begin{equation}
\label{eq:appx_two_pole_family}
Q_\alpha(z)
\;=\; \underbrace{\frac{1 - \beta_1}{1 - \beta_1 z^{-1}}}_{\text{momentum buffer}}
\cdot \underbrace{\frac{1}{1+\alpha}\left[ 1 + \frac{\alpha (1 - \beta_2)}{1 - \beta_2 z^{-1}} \right]}_{\text{second one-pole block}} ,
\end{equation}
so that $Q_\alpha(1) = 1$, the poles are $\beta_1$ and $\beta_2$, and there is one finite zero.
The parameter $\alpha$ interpolates between the two limits of interest: at $\alpha = 0$ the zero cancels the second pole and $Q_0$ is normalized Heavy Ball, while as $\alpha \to \infty$ the zero moves to the origin and $Q_\infty$ is the pure double-EMA cascade $Q_\infty(z) = \prod_{j=1,2} (1-\beta_j)/(1 - \beta_j z^{-1})$.

\begin{figure}[p]
\centering
\vspace*{\fill}
\includegraphics[width=\linewidth]{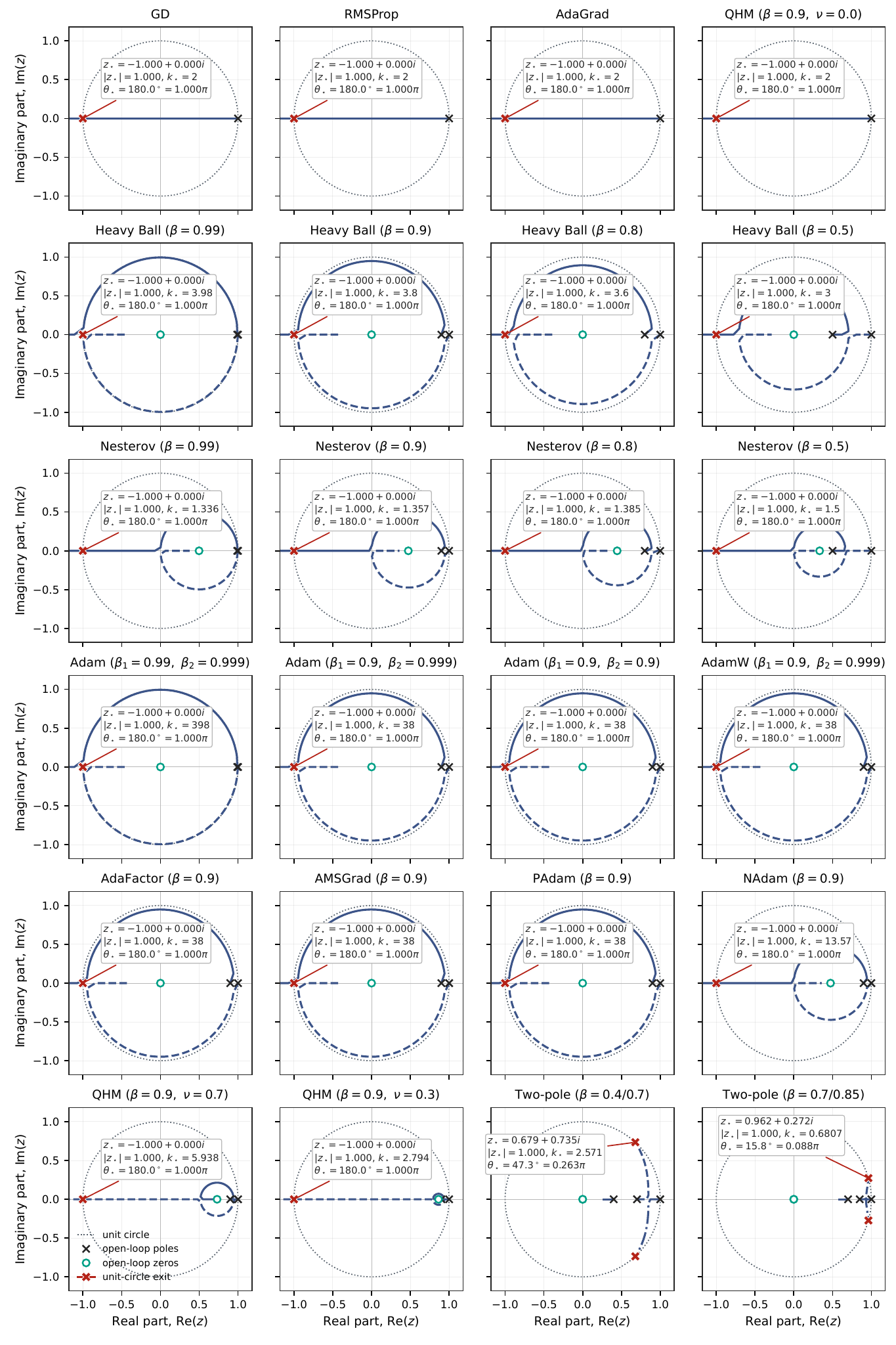}
\vspace{-1em}
\caption{\small
Complete root locus for the filters studied in this paper, extending Figure~\ref{fig:rootlocus} in the main text to all filters.}
\label{fig:rootlocusfull}
\vspace*{\fill}
\end{figure}

One EMA pole cannot reach a $90^\circ$ lag according to~(\ref{eq:appx_single_pole_margin_closed}).
A parallel sum stays inside the same cone, and no interior crossing is possible.
A cascade, by contrast, \emph{adds} lags, and two poles can jointly exceed the critical line.
Making this quantitative near $\omega = \pi$ already gives the exact threshold: writing $\omega = \pi - \epsilon$ and $\beta_1 = \beta_2 = \beta$, the cascade lag is $2\varphi_\beta(\pi-\epsilon) = 2\beta\epsilon/(1+\beta) + O(\epsilon^2)$ while the required lag is $\epsilon/2$, so an interior crossing detaches from $\omega = \pi$ if and only if
\begin{equation}
\label{eq:appx_third_threshold}
\frac{2\beta}{1+\beta} \;>\; \frac{1}{2}
\qquad\Longleftrightarrow\qquad
\beta \;>\; \tfrac{1}{3} .
\end{equation}
For two-pole momenta, we numerically evaluate the gauge and plot it in Figure~\ref{fig:rootlocusfull}.

\subsection{Stability Bound for Fixed Quadratic Loss}
\label{appx:proof:quadratic_toy}

To demonstrate that Proposition~\ref{prop:edge_preconditioner} yields the exact stability bound for a fixed quadratic loss, we conduct simple numerical experiments.
Four types of optimizers—GD, Adam, QHM, and cascade two-pole—are trained on a quadratic loss $f(x) = \tfrac12 x^2$ with various learning rates $\eta$.
Figure~\ref{fig:gaugetest} summarizes the results.
In the legend, we use the relative learning rate $\eta / \eta_{\text{r}}$, where $\eta_{\text{r}} = 2 / (S\Gamma)$ is determined by the stability condition $\chi = \eta S \Gamma / 2 \le 1$ in Proposition~\ref{prop:edge_preconditioner}.
For the two-pole momentum optimizer, we calculate the numerically found gauge, as shown in the root locus plot in Figure~\ref{fig:rootlocusfull}, which is not the Nyquist gain, i.e., $\Gamma \neq Q(-1)$.
Regardless of optimizer type, the stability bound is always satisfied on this fixed quadratic loss.

\begin{table}[t]
\centering
\footnotesize
\setlength{\tabcolsep}{3.5pt}
\caption{\small
Temporal filters of the executed families, in the paper-normalized convention $\chi=\eta S\Gamma/2$ of Appendix~\ref{appx:proof:gamma_table}.
$k_\star$ is the unit-curvature learning-rate edge.
$\omega_\star$ is the first unit-circle crossing in radians.
Adam, AdamW, Adafactor, AMSGrad and PAdam share one EMA filter; RMSProp,
AdaGrad and GD share the memoryless filter.
Note that two-pole momentum cases have $\Gamma \neq Q(-1)$.
}
\label{tab:appx_transfer}
\resizebox{\linewidth}{!}{
\begin{tabular}{@{}lcccccc@{}}
\toprule
Filter & Families & $Q(1)$ & $Q(-1)$ & $\Gamma$ & $k_\star$ & $\omega_\star/\pi$ \\
\midrule
GD
  & GD, RMSProp, AdaGrad
  & $1$ & $1$ & $1$ & $2$ & $1$ \\
EMA $\beta=0.9$
  & Adam, AdamW, Adafactor, AMSGrad, PAdam
  & $1$ & $0.05263$ & $0.05263$ & $38$ & $1$ \\
NAdam $\beta=0.9$
  & NAdam
  & $1$ & $0.1474$ & $0.1474$ & $13.57$ & $1$ \\
HB $\beta=0.9$
  & Heavy Ball
  & $10$ & $0.5263$ & $0.5263$ & $3.8$ & $1$ \\
NAG $\beta=0.9$
  & Nesterov
  & $10$ & $1.474$ & $1.474$ & $1.357$ & $1$ \\
QHM $\beta=0.9,\nu=0.7$
  & QHM
  & $1$ & $0.3368$ & $0.3368$ & $5.938$ & $1$ \\
two-pole $0.4/0.7$
  & Grokfast-style
  & $1$ & $0.07563$ & $0.7778$ & $2.571$ & $0.263$ \\
two-pole $0.7/0.85$
  & Grokfast-style
  & $1$ & $0.01431$ & $2.938$ & $0.681$ & $0.088$ \\
\bottomrule
\end{tabular}
}
\end{table}

\subsection{Complete Root Locus}
\label{appx:proof:complete_root_locus}

Complete root loci for the filters studied in this paper are shown in Figure~\ref{fig:rootlocusfull}, extending the demonstration in Figure~\ref{fig:rootlocus} in the main text to all filters.
See also Remark~\ref{rem:appx_positivity} in the previous section for a more detailed interpretation of the results.

\section{Stabilization from High-Order Terms}
\label{sec:third_order}
We begin by noting that the realized EoS~(\ref{eq:ray_stability_boundary}) is already exact.
Section~\ref{sec:theory} expresses the residual quantity in terms of a quadratic load $\vu^\top\bar{\mH}_{\eta,\vu}\vu$ and occasionally interprets this as an approximation of the second Taylor term $\bar{\mH}_{\eta,\vu}\simeq\mH$.
However, this simplification is only for our conceptual understanding of the quantity.
It does \emph{not} affect the exactness of our realized EoS $\zeta$.
This section provides a more detailed analysis to support this simplification.
Higher-order geometry does not explain \emph{why} the edge is occupied.
It determines \emph{how} an occupied quadratic load is realized as a finite step, and whether that realization rescues or amplifies a one-step crossing of $\zeta=1$.

\subsection{Exact nonlinear realization}
\label{appx:to:expansion}

Fix a step $\vtheta^+:=\vtheta-\eta\vu$ with $A:=\vg^\top\vu>0$ and $H_\vu:=\vu^\top\mH\vu>0$.
The quadratic and exact secant expenditures are defined as
\begin{equation}
\label{eq:appx_to_expenditures}
B_2 \;:=\; \frac{\eta}{2}\,H_\vu ,
\qquad
B
\;:=\;
A + \frac{\gL(\vtheta^+)-\gL(\vtheta)}{\eta}
\;=\;
\frac{\eta}{2}\,\vu^\top\bar{\mH}_{\eta,\vu}\vu ,
\end{equation}
the second equality being~(\ref{eq:curvature_load_chain_rule}).
We also define the corresponding loads $\zeta_2:=B_2/A$ and $\zeta:=B/A$.
Let the \emph{nonlinear realization factor} be the ratio of the two,
\begin{equation}
\label{eq:appx_to_h}
h \;:=\; \frac{\zeta}{\zeta_2} \;=\; \frac{B}{B_2} ,
\end{equation}
so that, identically,
\begin{equation}
\label{eq:appx_to_factorization}
\zeta \;=\; \zeta_2\, h .
\end{equation}
Note that the spatial factor $\varsigma$ of~(\ref{eq:optimizer_aware_eos_factor}) is built from the \emph{secant} Hessian, hence already contains $h$.
Write
\begin{equation}
\label{eq:appx_to_varsigma2}
\varsigma_2
\;:=\;
\frac{\vu^\top\mH\vu}{S\|\vu\|_{\mP}^2}
\end{equation}
for the quadratic occupancy of the same ray.
Then $\varsigma=\varsigma_2 h$, and the main-text factorization splits as
\begin{equation}
\label{eq:appx_to_hierarchy}
\zeta_2 \;=\; \chi\,\frac{\varsigma_2}{\tau} ,
\qquad
\zeta
\;=\; \chi\,\frac{\varsigma_2}{\tau}\,h .
\end{equation}
Thus we can further interpret that (1) $\chi$ is \emph{available load}, (2) $\varsigma_2/\tau$ occupies a \emph{quadratic-only load} along $\vu$, and (3) $h$ \emph{realizes} that load as a finite step.
We need not apply the correction factor $h$ to the main-text $\varsigma$ again, which is already exact as $\varsigma=\varsigma_2 h$.
Multiplying the main-text $\varsigma$ by a further $h$ would double-count the same correction.

Since we are dealing with directional derivatives along the update direction $\vu$ actuated by the optimizer, the factor $h$ is a Taylor series of one \emph{scalar}, not a collection of tensors.
Write $T_\vu:=\nabla^3\gL[\vu]^3$, $F_\vu:=\nabla^4\gL[\vu]^4$, $P_\vu:=\nabla^5\gL[\vu]^5$, and
\begin{equation}
\label{eq:appx_to_thetas}
\theta_3 := \frac{\eta\, T_\vu}{3 H_\vu} , \qquad
\theta_4 := \frac{\eta^2 F_\vu}{12 H_\vu} , \qquad
\theta_5 := \frac{\eta^3 P_\vu}{60 H_\vu} .
\end{equation}

\begin{lemma}[Secant Hessian expansion]
\label{lem:appx_to_secant}
Let $\gL$ be $C^{m+2}$ on a neighborhood of the executed segment $\{\vtheta-s\eta\vu:s\in[0,1]\}$.
(If the activation is piecewise linear, this requires that the segment not cross a kink.)
Then the weighted secant Hessian~(\ref{eq:ray_parameterization}) satisfies
\begin{equation}
\label{eq:appx_to_secant_series}
\bar{\mH}_{\eta,\vu}
\;=\;
2\sum_{k=0}^{m} \frac{(-\eta)^k}{(k+2)!}\, \nabla^{k+2}\gL(\vtheta)[\vu]^k
\;+\; O(\eta^{m+1}) ,
\end{equation}
where $\nabla^{j}\gL[\vu]^{k}$ denotes the $j$-th derivative tensor contracted $k$ times with $\vu$.
\end{lemma}

\begin{proof}
Expand $\nabla^2\gL(\vtheta-s\eta\vu)=\sum_{k\ge0}\frac{(-s\eta)^k}{k!}\nabla^{k+2}\gL(\vtheta)[\vu]^k$ and integrate against the kernel $2(1-s)$ using $2\int_0^1(1-s)s^k\,ds=\frac{2}{(k+1)(k+2)}$.
The coefficient of $\nabla^{k+2}\gL[\vu]^k$ is therefore $\frac{2(-\eta)^k}{k!\,(k+1)(k+2)}=\frac{2(-\eta)^k}{(k+2)!}$.
\end{proof}

\begin{proposition}[Impedance series and its loss-only form]
\label{prop:appx_to_series}
With $A\neq 0$,
\begin{equation}
\label{eq:appx_to_zeta_series}
\zeta
\;=\;
\sum_{j\ge 2} \frac{(-1)^j \eta^{j-1}}{j!\,A}\, \nabla^j \gL(\vtheta)[\vu]^j
\;=\;
\zeta_2\bigl(1-\theta_3+\theta_4-\theta_5+O(\eta^4)\bigr) ,
\end{equation}
and therefore $h=1-\theta_3+\theta_4-\theta_5+O(\eta^4)$.
Moreover, with $\Delta\gL:=\gL(\vtheta-\eta\vu)-\gL(\vtheta)$,
\begin{equation}
\label{eq:appx_to_loss_form}
\zeta-1
\;=\;
\frac{\Delta\gL}{\eta A}
\;=\;
\frac{\gL(\vtheta-\eta\vu)-\gL(\vtheta)}{\eta\,\vg^\top\vu} ,
\end{equation}
so~(\ref{eq:appx_to_zeta_series}) expands a quantity that is exactly the one-step loss change normalized by the learning power.
\end{proposition}

\begin{proof}
The first equality follows from Lemma~\ref{lem:appx_to_secant} and $\zeta=\frac{\eta}{2A}\vu^\top\bar{\mH}_{\eta,\vu}\vu$.
The factorization is~(\ref{eq:appx_to_thetas}) collected term by term.
For~(\ref{eq:appx_to_loss_form}), expand $\Delta\gL=\sum_{j\ge1}\frac{(-\eta)^j}{j!}\nabla^j\gL[\vu]^j=-\eta A+\sum_{j\ge2}\frac{(-\eta)^j}{j!}\nabla^j\gL[\vu]^j$ and divide by $\eta A$.
Equivalently, it is~(\ref{eq:curvature_load_chain_rule}) restated.
\end{proof}

Equation~(\ref{eq:appx_to_loss_form}) is Hessian-free, which was mentioned in the main text as well.
Only two forward passes and the inner product $A$ are needed, which are already computed by the optimizer.
Only the quadratic load $\zeta_2$ requires a Hessian-vector product.
The series~(\ref{eq:appx_to_zeta_series}) is simply the Taylor expansion of the exact ratio $h$, not an independent mechanism.
The following corollary upper bounds the approximation error when we treat the secant Hessian as the true Hessian.

\begin{corollary}[Approximation of the edge of stability]
\label{cor:optimizer_aware_eos_approximation}
If the Hessian $\mH=\nabla_{\vtheta}^2\gL(\vtheta)$ is $L_\mH$-Lipschitz on the segment
$\{\vtheta-s\eta\vu:s\in[0,1]\}$, then
\begin{equation}
\label{eq:optimizer_aware_eos_approximation_bound}
\left| \vu^\top (\bar{\mH}_{\eta, \vu} - \mH) \vu \right| \;\leq\; \frac{L_\mH \eta \|\vu\|^3}{3}.
\end{equation}
Writing
\begin{equation}
\label{eq:optimizer_aware_eos_approximation_estimation}
\bar\zeta \;:=\; \frac{\eta \vu^\top \mH \vu}{2 \vg^\top \vu}
\end{equation}
for the quadratic estimator (i.e.\ $\zeta_2$ of~(\ref{eq:appx_to_zeta_series})), the approximation error is bounded by
\begin{equation}
\label{eq:optimizer_aware_eos_approximation_error_bound}
\left| \zeta - \bar \zeta \right| \;\leq\; \frac{L_\mH \eta^2 \|\vu\|^3}{6 \,\lvert \vg^\top \vu \rvert}.
\end{equation}
\end{corollary}

\begin{proof}
By~(\ref{eq:ray_parameterization}), $\vu^\top(\bar{\mH}_{\eta,\vu}-\mH)\vu = 2\int_0^1 (1-s)\, \vu^\top\bigl(\nabla^2\gL(\vtheta-s\eta\vu)-\mH\bigr)\vu\,ds$.
Lipschitz continuity of $\mH$ bounds the integrand by $L_\mH (s\eta\|\vu\|)\,\|\vu\|^2$, so the integral is at most $2 L_\mH \eta \|\vu\|^3 \int_0^1 (1-s)s\,ds = L_\mH \eta \|\vu\|^3 / 3$.
Dividing by $2\lvert A \rvert / \eta$ yields~(\ref{eq:optimizer_aware_eos_approximation_error_bound}).
This is the $m=0$ remainder of Lemma~\ref{lem:appx_to_secant} under a Lipschitz rather than a $C^3$ hypothesis.
\end{proof}

All higher-order content in $\zeta$ is this one scalar.
It is the ratio of the averaged directional curvature over the step to the instantaneous directional curvature at its start.
The coefficients $-\theta_3+\theta_4-\cdots$ are its Taylor coefficients.
This scalar can act in either side of the edge.

Let $\hat\vu:=\vu/\|\vu\|$, let $r:=\eta\|\vu\|$ be the step length, and define the \emph{directional curvature profile}
\begin{equation}
\label{eq:appx_to_profile_defs}
\psi(s) \;:=\; \hat\vu^\top \nabla^2\gL(\vtheta - s\hat\vu)\,\hat\vu ,
\qquad
g_\vu \;:=\; \vg^\top\hat\vu \;>\; 0 .
\end{equation}

\begin{proposition}[Profile form]
\label{prop:appx_to_profile}
The impedance is the $(r-s)$-weighted mean of the directional curvature over the executed step,
\begin{equation}
\label{eq:appx_to_profile}
\zeta \;=\; \frac{1}{r\,g_\vu}\int_0^r (r-s)\,\psi(s)\,ds \;=\; \frac{r\,\bar\psi_r}{2\,g_\vu},
\qquad
\bar\psi_r := \frac{2}{r^2}\int_0^r (r-s)\,\psi(s)\,ds ,
\end{equation}
and therefore, with $\zeta_2 = r\psi(0)/(2g_\vu)$,
\begin{equation}
\label{eq:appx_to_master_ratio}
h \;=\; \frac{\zeta}{\zeta_2} \;=\; \frac{\bar\psi_r}{\psi(0)} .
\end{equation}
\end{proposition}

\begin{proof}
Let $\phi(s):=\gL(\vtheta-s\hat\vu)$, so $\phi'(0)=-g_\vu$ and $\phi''(s)=\psi(s)$.
Taylor's theorem with integral remainder gives
$\phi(r)-\phi(0)=-r g_\vu+\int_0^r(r-s)\psi(s)\,ds$, and $\eta A=r g_\vu$, so
by~(\ref{eq:appx_to_loss_form}) $\zeta=1+\frac{\phi(r)-\phi(0)}{r g_\vu}$, which
is~(\ref{eq:appx_to_profile}).
\end{proof}

\begin{corollary}[Monotone softening and stiffening]
\label{cor:appx_to_monotone}
If $\psi$ is nonincreasing on $[0,r]$ then $h\le 1$, and if $\psi$ is nondecreasing then $h\ge 1$.
Both inequalities are strict if the monotonicity is strict on a set of positive measure.
More quantitatively, if $\psi(s)\le(1-\mu)\psi(0)$ for $s\ge\varrho r$ with $\mu\in(0,1)$ and
$\varrho\in(0,1)$, then
\begin{equation}
\label{eq:appx_to_softening_bound}
h \;\le\; 1 - \mu (1-\varrho)^2 .
\end{equation}
\end{corollary}

\begin{proof}
The kernel $2(r-s)/r^2$ is a probability density on $[0,r]$, so $\bar\psi_r$ is an average of $\psi$
and~(\ref{eq:appx_to_master_ratio}) gives the two monotone statements.
For the bound, $\int_{\varrho r}^{r} \frac{2(r-s)}{r^2}\,ds=(1-\varrho)^2$, and on that sub-interval
$\psi$ is below $(1-\mu)\psi(0)$ while elsewhere it is at most $\psi(0)$.
\end{proof}

Thus we can call $h<1$ \emph{directional softening} and $h>1$ \emph{directional stiffening}.

\subsection{Entry, return, and maintenance}
\label{appx:to:crossing}

Because $\zeta=\zeta_2 h$, a step lies below, on, or above the realized edge according as $h$
lies below, on, or above $1/\zeta_2$:
\begin{equation}
\label{eq:appx_to_crossing}
\zeta<1
\;\;\Longleftrightarrow\;\;
h<\frac{1}{\zeta_2} ,
\qquad
\zeta=1
\;\;\Longleftrightarrow\;\;
h=\frac{1}{\zeta_2} ,
\qquad
\zeta>1
\;\;\Longleftrightarrow\;\;
h>\frac{1}{\zeta_2} .
\end{equation}
These are simple identities.
When we are discussing the higher-order term $h$, we are mainly interested in the condition whether this term rescues or destabilizes a one-step crossing of $\zeta=1$, i.e., whether $h$ is below, on, or above $1/\zeta_2$.

\begin{theorem}[Nonlinear rescue and destabilization]
\label{thm:appx_to_below}
Assume $A>0$ and $H_\vu>0$.
\begin{enumerate}
\item If $\zeta_2=1+\delta$ with $\delta>0$, the step is quadratically supercritical.
It is rescued to $\zeta\le 1$ if and only if
\begin{equation}
\label{eq:appx_to_below_condition}
h \;\le\; \frac{1}{1+\delta}
\qquad\text{i.e.}\qquad
1-h \;\ge\; \frac{\delta}{1+\delta} .
\end{equation}
To the order retained in~(\ref{eq:appx_to_zeta_series}), the same comparison reads $\theta_3-\theta_4+O(\eta^3)\ge\delta/(1+\delta)$, with remainder controlled by Corollary~\ref{cor:optimizer_aware_eos_approximation} when only the Lipschitz modulus is assumed.
\item If $\zeta_2=1-\delta$ with $\delta\in(0,1)$, the step is quadratically subcritical.
Nonlinear stiffening pushes it above the edge if and only if $h>1/(1-\delta)$.
\item If $\psi$ is nonincreasing and $\zeta_2\le 1$, then $\zeta\le\zeta_2\le 1$.
If $\psi$ is nondecreasing and $\zeta_2\ge 1$, then $\zeta\ge\zeta_2\ge 1$.
\end{enumerate}
\end{theorem}

\begin{proof}
Claims 1--2 are~(\ref{eq:appx_to_crossing}) rewritten in $\delta$.
Claim 3 is Corollary~\ref{cor:appx_to_monotone}.
\end{proof}

A quadratically supercritical step is therefore rescued precisely when the fractional drop in path-averaged directional curvature exceeds the fractional linear excess.
A quadratically subcritical step is destabilized when stiffening consumes more than the quadratic margin.
We only need to consider two event indicators, both measurable from $(\zeta_2,\zeta)$ alone:
\emph{nonlinear rescue} is $\zeta_2>1$ and $\zeta\le 1$;
\emph{nonlinear destabilization} is $\zeta_2\le 1$ and $\zeta>1$.
An $\epsilon$-band $1-\epsilon\le\zeta\le 1+\epsilon$ is equivalent to $(1-\epsilon)/\zeta_2\le h\le(1+\epsilon)/\zeta_2$.
That is a one-step constraint on $h$, not a statement that a period-$2$ orbit remains in $1\le\zeta<1+\epsilon$ (a nontrivial two-cycle with $A_t>0$ must change the sign of $\zeta-1$).

The factorization~(\ref{eq:appx_to_hierarchy}) also splits \emph{changes} of $\zeta$.
On the positive domain,
\begin{equation}
\label{eq:appx_to_log_split}
\log\zeta
\;=\;
\log\chi + \log\varsigma_2 - \log\tau + \log h ,
\end{equation}
hence
\begin{equation}
\label{eq:appx_to_log_drift}
\Delta\log\zeta
\;=\;
\Delta\log\chi + \Delta\log\varsigma_2 - \Delta\log\tau + \Delta\log h .
\end{equation}
Basically, this is the same as the one in the main text, but with one additional granularity, decomposing the spatial participation $\varsigma$ into the quadratic participation $\varsigma_2$ and higher-order maintenance $h$.
Each summand is a channel of diagnostics, we may not treat them as causes while analyzing the dynamics.
Therefore, we can read the diagnostics by following interpretations:
\begin{center}
\small
\begin{tabular}{@{}lll@{}}
\toprule
Channel & raises $\zeta$ & lowers $\zeta$ \\
\midrule
$\chi$ & spectral sharpening & spectral relief \\
$\varsigma_2$ & occupancy of a sharper direction & rotation toward lower curvature \\
$\tau$ & loss of temporal calibration & recovery of alignment/calibration \\
$h$ & pathwise stiffening & pathwise softening \\
\bottomrule
\end{tabular}
\end{center}
Entry from below ($\zeta_t<1$ and $\Delta\log\zeta_t>0$) is $\Delta\log\chi+\Delta\log\varsigma_2+\Delta\log h>\Delta\log\tau$, 
return from above is the opposite inequality $\Delta\log\chi+\Delta\log\varsigma_2+\Delta\log h<\Delta\log\tau$, and
maintenance near $\zeta\simeq 1$ is approximate balance of the four increments.
These are measurable channel-balance conditions.
Note that these are not an attractor theorem: the quantities do not assert that higher-order terms should drive $\zeta$ toward $1$, nor that $h$ is the dominant channel.

\section{Additional Experiments}
\label{sec:experiment}
This section presents the experimental setup and all experimental results.

\subsection{Implementation Details}
\label{appx:protocol}

\begin{table}[t]
\centering
\footnotesize
\setlength{\tabcolsep}{2.5pt}
\caption{\small
Optimizer grids on used in the paper.
Learning rates are five geometrically spaced values.
The displayed endpoints are the grid minima and maxima.
``BC'' is bias correction.
$p$ is the preconditioner exponent on the second moment.
$\mathrm{wd}$ is decoupled weight decay.}
\label{tab:appx_hyperparams}
\begin{tabular}{@{}lcccccc@{}}
\toprule
Family & $\eta$ grid & $\beta_1$ / $\beta$ & $\beta_2$ & $\varepsilon$ & BC / $p$ & $\mathrm{wd}$ \\
\midrule
Adam
  & $3\times 10^{-5}$--$10^{-3}$
  & $0.9$ & $0.999$ & $10^{-7}$ & yes / $0.5$ & $0$ \\
AdamW
  & $3\times 10^{-5}$--$10^{-3}$
  & $0.9$ & $0.999$ & $10^{-7}$ & yes / $0.5$ & $0.01$ \\
Adafactor
  & $10^{-5}$--$10^{-3}$
  & $0.9$ & $0.8$ & $10^{-7}$ & no / $0.5$ & $0$ \\
AMSGrad
  & $10^{-5}$--$10^{-3}$
  & $0.9$ & $0.999$ & $10^{-7}$ & no / $0.5$ & $0$ \\
PAdam
  & $10^{-3}$--$10^{-1}$
  & $0.9$ & $0.999$ & $10^{-7}$ & no / $0.25$ & $0$ \\
NAdam
  & $10^{-5}$--$10^{-3}$
  & $0.9$ & $0.999$ & $10^{-7}$ & no / $0.5$ & $0$ \\
RMSProp
  & $10^{-4}$--$10^{-3}$
  & $0$ & $0.995$ & $10^{-7}$ & no / $0.5$ & $0$ \\
AdaGrad
  & $10^{-2}$--$10^{-1}$
  & $0$ & $1$ & $10^{-10}$ & no / $0.5$ & $0$ \\
GD
  & $0.02$--$0.2$
  & --- & --- & --- & --- & $0$ \\
Heavy Ball
  & $0.002$--$0.02$
  & $0.9$ & --- & --- & --- & $0$ \\
Nesterov
  & $0.002$--$0.02$
  & $0.9$ & --- & --- & --- & $0$ \\
QHM
  & $0.059375$--$0.59375$
  & $0.9$ & --- & --- & $\nu=0.7$ & $0$ \\
two-pole $0.4/0.7$
  & $0.0257$--$0.257$
  & --- & --- & --- & cascade & $0$ \\
two-pole $0.7/0.85$
  & $0.00681$--$0.0681$
  & --- & --- & --- & cascade & $0$ \\
\bottomrule
\end{tabular}
\end{table}

\paragraph{Dataset.}
We use CIFAR-10~\citep{krizhevsky2009learning} dataset, all training samples are used.
Every update uses one deterministic $50{,}000$-example full batch without shuffling.

\paragraph{Architecture.}
The network flattens the $3\times 32\times 32$ input, applies five fully connected hidden layers of width $200$ with $\tanh$ activations, and produces $10$ logits.
The objective is mean categorical cross-entropy on the full training set.
This is exactly the same as the $5\times 200$ architecture of \citet{cohen2023adaptive}, not the $2\times 200$ network of \citet{cohen2021gradient}.

\paragraph{Optimizers and their hyperparameters.}
Table~\ref{tab:appx_hyperparams} displays the learning rate grids and optimizer hyperparameters.
AdamW weight decay is $0.01$ and AdaGrad $\varepsilon$ is $10^{-10}$.
Adam uses bias correction.
Adafactor, AMSGrad, PAdam, NAdam, RMSProp and AdaGrad do not.
PAdam uses partial-adaptivity exponent $0.25$.
QHM uses $\beta=0.9$, $\nu=0.7$.
The two-pole filters are cascaded EMA poles at $(0.4,0.7)$ and $(0.7,0.85)$, i.e., Grokfast-style~\citep{lee2024grokfast} cascades rather than parallel mixtures (Appendix~\ref{appx:proof:two_pole}).

\paragraph{Temporal filters $Q$.}
The $14$ optimizer names collapse to eight distinct causal filters $Q$, whose transfer functions are calculated in Table~\ref{tab:appx_transfer}.
Heavy Ball and Nesterov are the classical (not dc-normalized) recurrences, so $Q(1)=1/(1-\beta)=10$ at $\beta=0.9$.
Adam-family first moments are dc-normalized and share the EMA filter of the dc-normalized Heavy Ball.
The two-pole rows are the only filters whose first root-locus exit is off Nyquist.
Refer to Figure~\ref{fig:rootlocusfull} for the full root locus.

\begin{table}[t]
\centering
\small
\caption{\small
Measured relative edge $\chi=\eta S\Gamma/2$ offsets for different optimizers.
Each cell is the median of $\chi$ between first entry above $0.75$ and last exit below $0.75$ (open windows remain open at the horizon).
A dash means the trajectory never exceeded $0.75$.
Extreme cases $\chi > 1.2$ are highlighted in boldface.
$^*$The manuscript marks AdaGrad and RMSProp as sparse edge-touch statistics.}
\label{tab:offsets}
\resizebox{\linewidth}{!}{%
\begin{tabular}{@{}lrrrrrrrll@{}}
\toprule
Family & $\eta_1$ & $\eta_2$ & $\eta_3$ & $\eta_4$ & $\eta_5$
& Center & $n$ & $\chi$ range & $\zeta$ range \\
\midrule
Adafactor & \textemdash & 1.052 & 1.103 & 1.102 & 1.191 & 1.084 & 204/350 & 1.052--1.191 & $0.202442\text{--}0.965959$ \\
AdaGrad & 1.424 & 1.556 & 1.845 & 4.047 & 21.070 & $1.248^*$ & 3/350 & 1.424--21.070 & $0.982240\text{--}1.011240$ \\
Adam & 0.927 & 0.980 & 0.971 & 0.976 & 0.974 & 0.960 & 259/350 & 0.927--0.980 & $0.177652\text{--}0.948673$ \\
AdamW & 0.956 & 0.981 & 0.975 & 1.023 & 0.971 & 0.971 & 247/350 & 0.956--1.023 & $0.115807\text{--}0.936750$ \\
AMSGrad & 0.789 & 0.931 & 0.979 & 0.946 & 0.949 & 0.960 & 203/350 & 0.789--0.979 & $-0.000206\text{--}0.836529$ \\
GD & 0.904 & 1.130 & 1.246 & 1.478 & 1.624 & 1.124 & 47/140 & 0.904--1.624 & $0.000101\text{--}0.991144$ \\
Heavy Ball & \textemdash & 0.820 & 0.960 & 0.934 & 0.964 & 0.949 & 165/400 & 0.820--0.964 & $0.000313\text{--}0.887029$ \\
NAdam & 1.036 & 1.050 & 0.884 & 1.183 & 1.198 & 1.070 & 259/350 & 0.884--1.198 & $0.087206\text{--}0.971799$ \\
Nesterov & 0.836 & 1.078 & 1.088 & 1.123 & 1.121 & 1.060 & 206/400 & 0.836--1.123 & $-0.000020\text{--}0.977996$ \\
PAdam & 0.895 & 0.959 & 0.966 & 0.927 & 0.950 & 0.956 & 236/350 & 0.895--0.966 & $-0.000206\text{--}0.857368$ \\
QHM (0.9, 0.7) & 0.961 & 1.238 & 1.309 & 1.454 & 1.576 & 1.100 & 116/350 & 0.961--1.576 & $0.000287\text{--}1.001553$ \\
RMSProp & 1.349 & 1.452 & 1.568 & 1.579 & 1.646 & $1.222^*$ & 12/350 & 1.349--1.646 & $0.978233\text{--}1.011893$ \\
Grokfast (0.4/0.7) & 0.982 & 1.432 & 1.819 & 2.294 & 2.646 & 1.021 & 52/350 & 0.982--2.646 & $0.000152\text{--}0.964556$ \\
Grokfast (0.7/0.85) & \textemdash & 0.988 & 1.354 & 1.625 & 2.049 & 1.040 & 61/350 & 0.988--2.049 & $0.000365\text{--}0.476646$ \\
\bottomrule
\end{tabular}
}
\end{table}

\subsection{Full Comparison of Worst-Case EoS $\chi$ and Realized EoS $\zeta$}
\label{appx:comparison_chi_zeta}

This section shows a full comparison of worst-case EoS $\chi$ and realized EoS $\zeta$, extending Figure~\ref{fig:observation} in the main text.
Figures~\ref{fig:eos-1-3} and Figures~\ref{fig:eos-4-6} show the results.
Table~\ref{tab:offsets} is the complete numerical results, measured as the median of $\chi$ on the closed interval from first $\chi>0.75$ to last $\chi<0.75$, or to the terminal checkpoint if the trajectory never returns below $0.75$.
Two settings never exceed $0.75$ (Adafactor $\eta=10^{-5}$, Heavy Ball $\eta=0.002$, two-pole $0.7/0.85$ at $\eta=0.00681$) and are recorded as dashes.
Besides the discussion in Section~\ref{sec:observation}, Table~\ref{tab:offsets} also suggests that the offsets are not an artefact of decoupled weight decay.
Adam ($\lambda=0$) and AdamW ($\lambda=0.01$) share the same temporal filter, the same $\eta$ grid, and the same architecture; their edge-conditioned family centers are $0.960$ and $0.971$.

\begin{figure}[p]
    \centering
    \vspace*{\fill}
    \includegraphics[width=\linewidth]{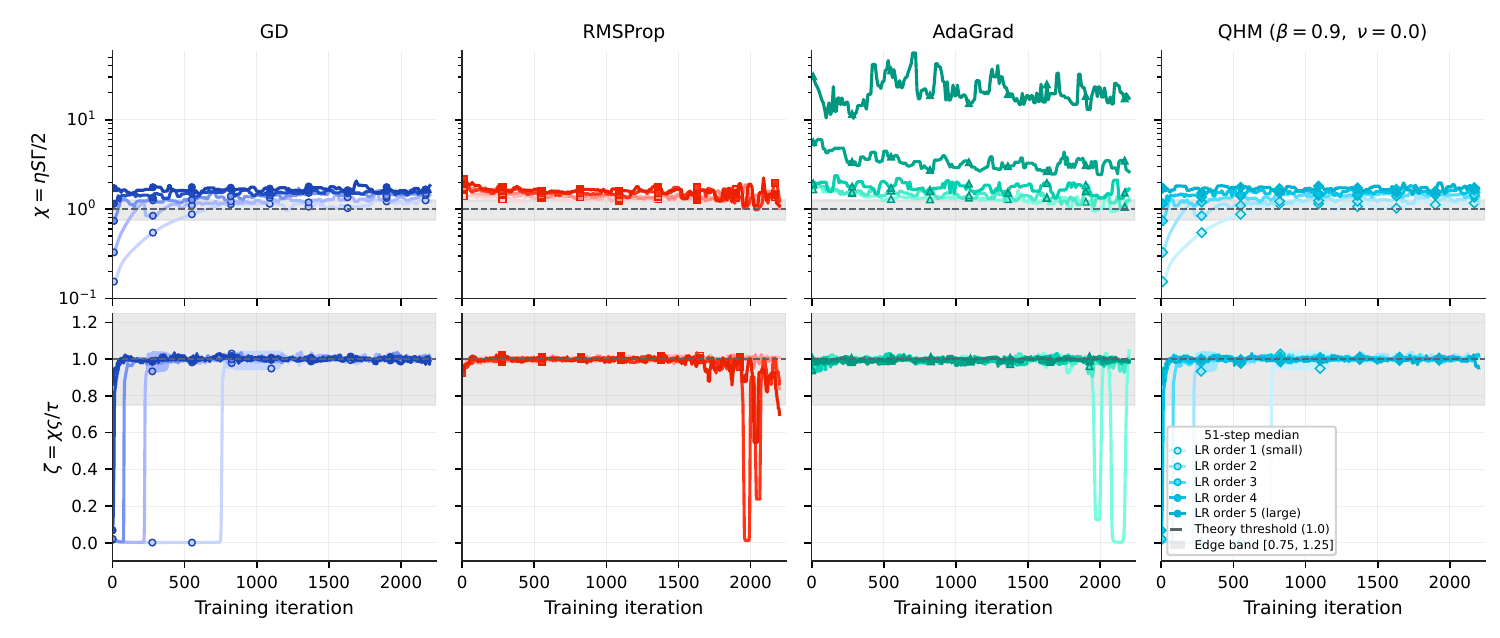}
    \includegraphics[width=\linewidth]{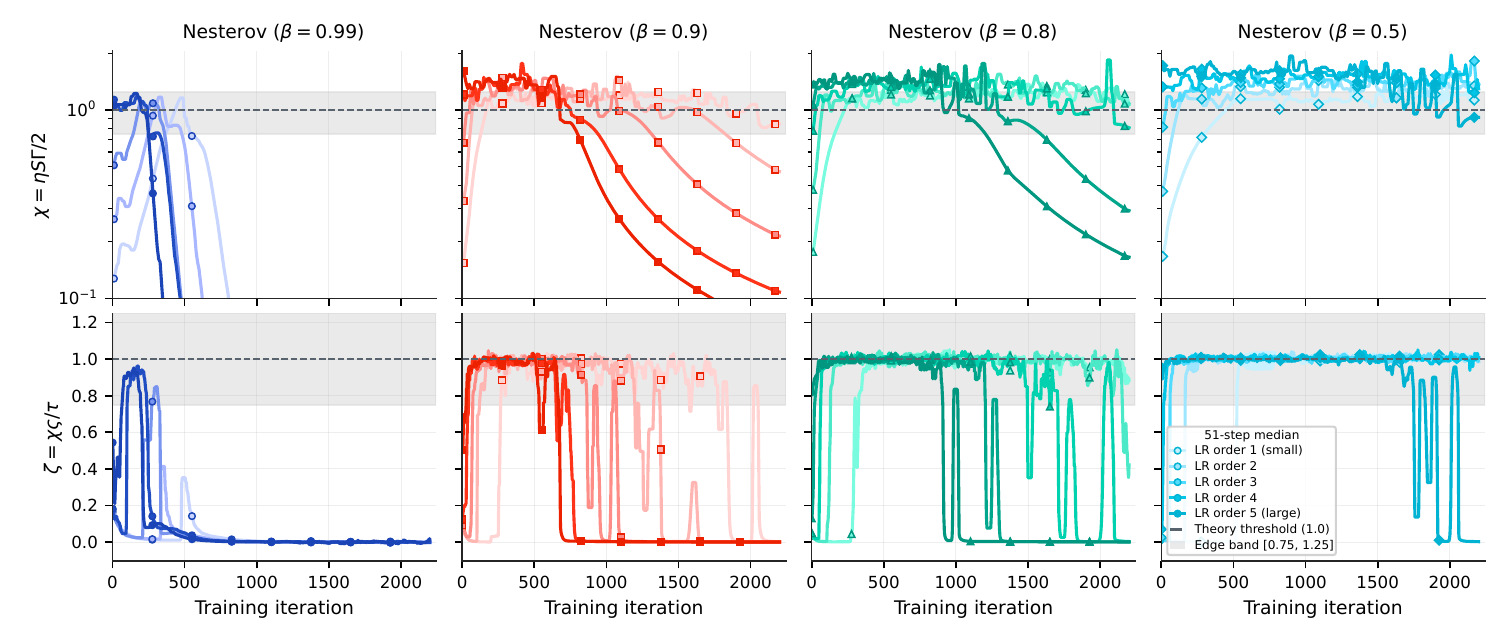}
    \includegraphics[width=\linewidth]{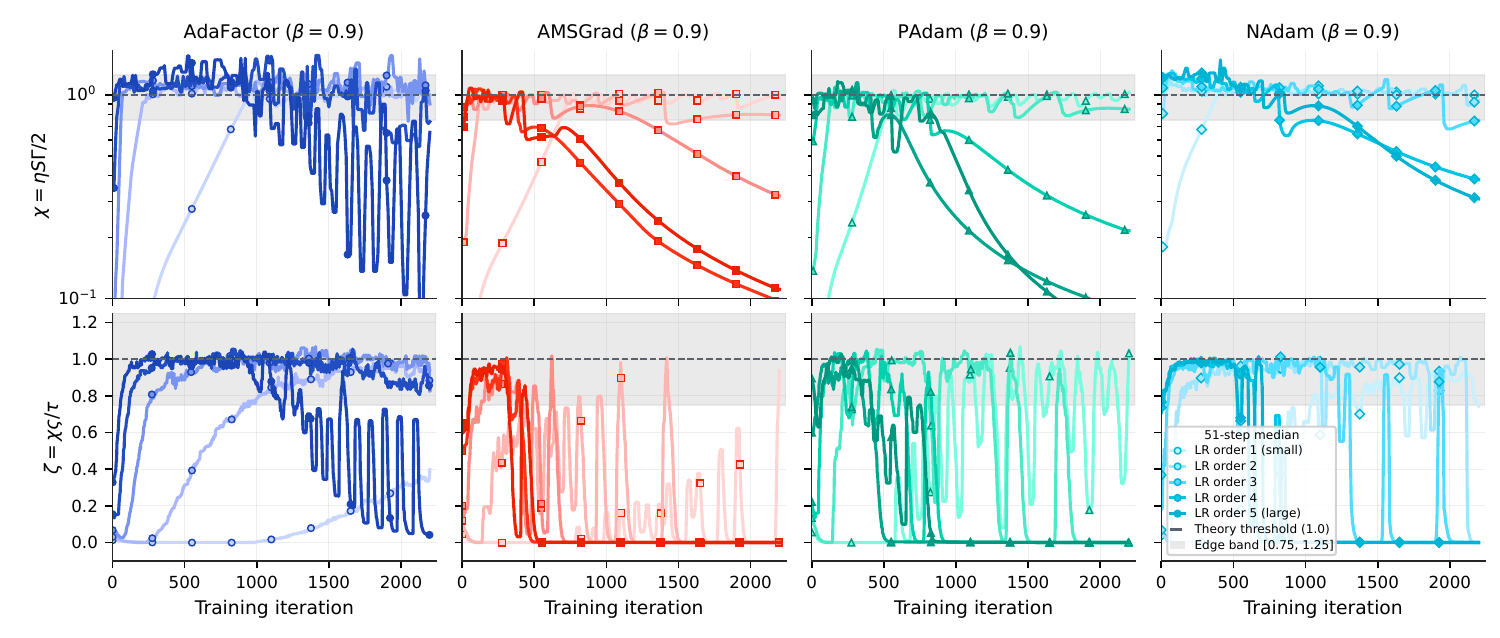}
    \vspace{-1em}
    \caption{\small
    Comparison of worst-case EoS $\chi$ and realized EoS $\zeta$ on the primary CIFAR-10 experiments.}
    \label{fig:eos-1-3}
    \vspace*{\fill}
\end{figure}

\begin{figure}[p]
    \centering
    \vspace*{\fill}
    \includegraphics[width=\linewidth]{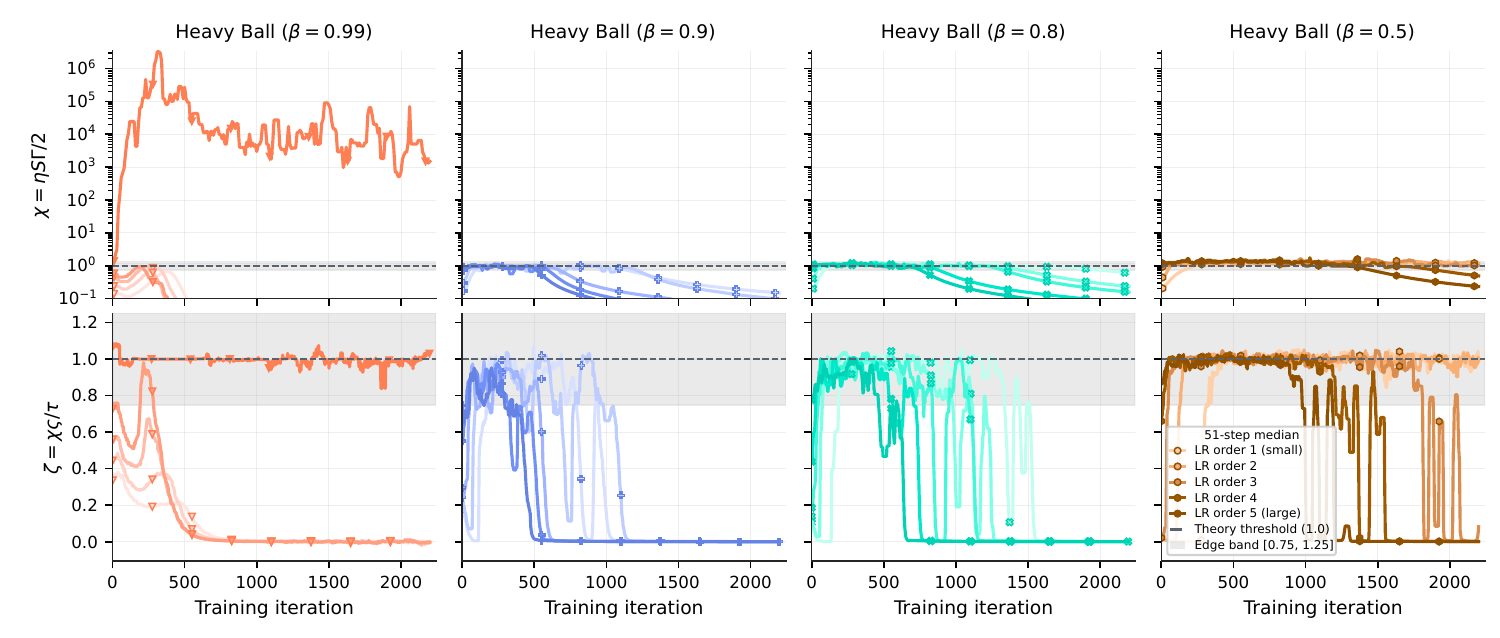}
    \includegraphics[width=\linewidth]{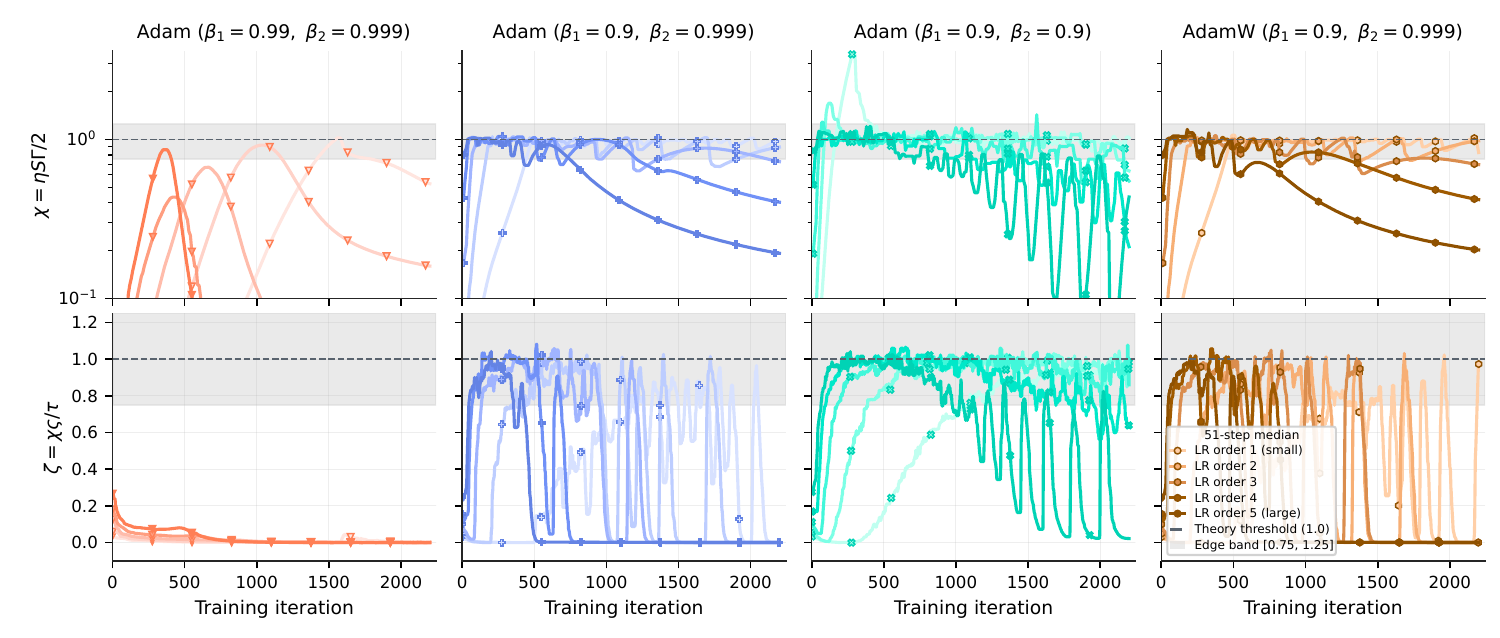}
    \includegraphics[width=\linewidth]{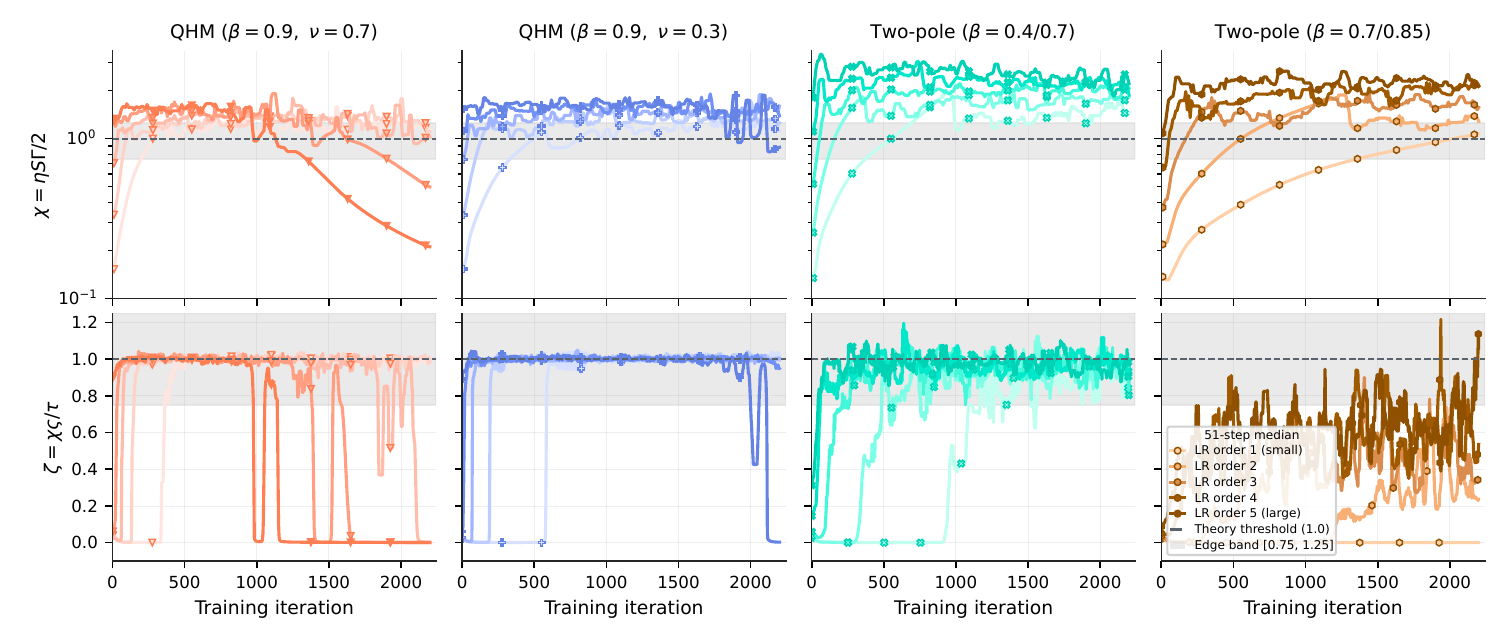}
    \vspace{-1em}
    \caption{\small
    Comparison of worst-case EoS $\chi$ and realized EoS $\zeta$ on the primary CIFAR-10 experiments.}
    \label{fig:eos-4-6}
    \vspace*{\fill}
\end{figure}

\subsection{Full Decomposition of Realized EoS $\zeta$}
\label{appx:decomposition_zeta}

This section shows a full decomposition of realized EoS $\zeta$ for each individual setting of (optimizer, learning rate), extending Figure~\ref{fig:decomposition_of_learning_impedance} in the main text.
Full results are visualized in Figures~\ref{fig:decomposition-1} through Figure~\ref{fig:decomposition-4}.

\begin{figure}[p]
\centering
\vspace*{\fill}
\includegraphics[width=\linewidth]{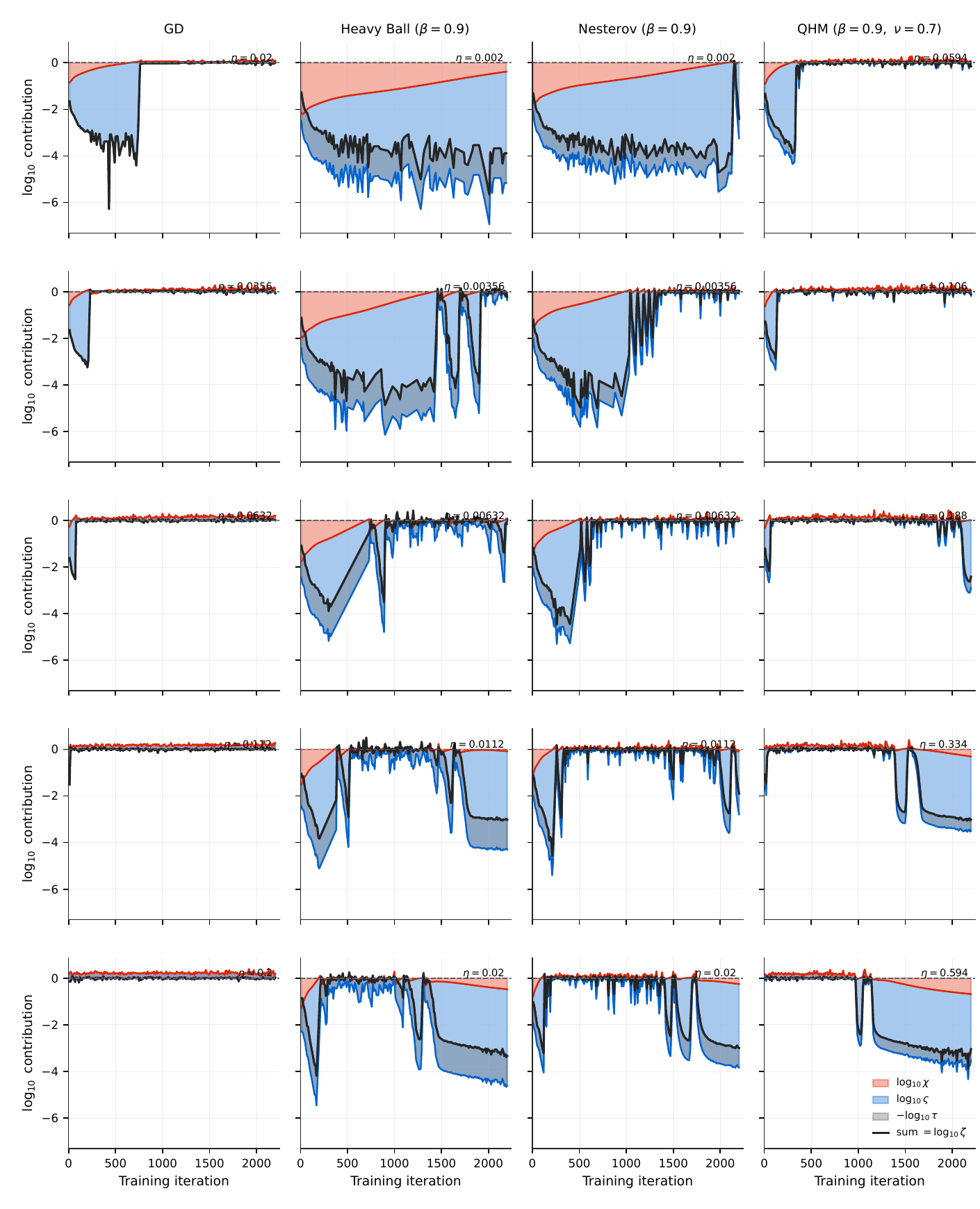}
\vspace{-1em}
\caption{\small
Decomposition of the realized EoS $\zeta$ into factors in log-scale, involving the (relative) edge of stability $\chi$ ({\color{red!50} red} area and line), the spatial participation factor $\varsigma$ ({\color{blue!50} blue} area and line), and the temporal calibration factor $\tau$ ({\color{white!30!black} grey} overlay and black line).
Note that $\log \tau$ operates in the opposite direction by subtraction.
}
\label{fig:decomposition-1}
\vspace*{\fill}
\end{figure}

\begin{figure}[p]
\centering
\vspace*{\fill}
\includegraphics[width=\linewidth]{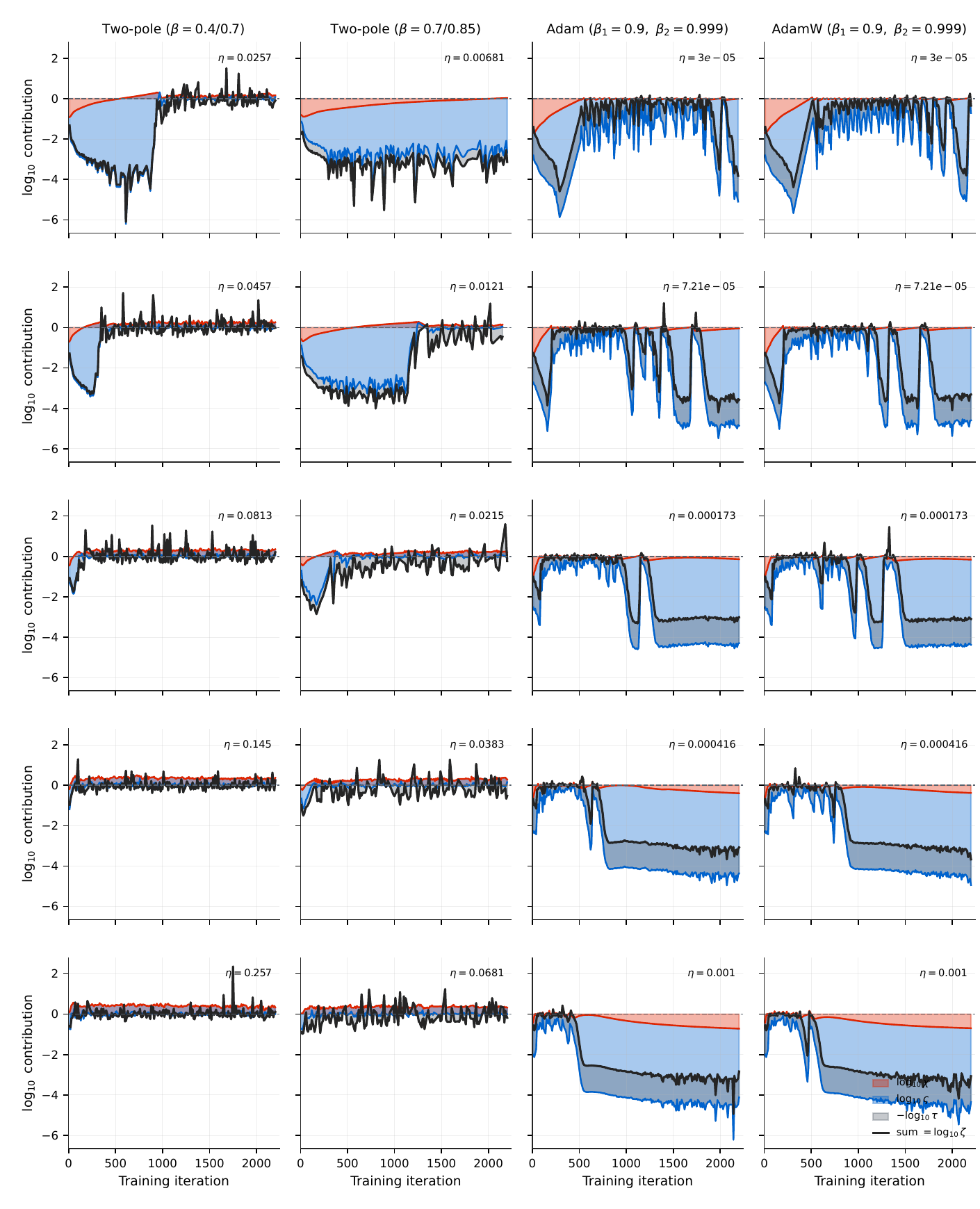}
\vspace{-1em}
\caption{\small
Decomposition of the realized EoS $\zeta$ into factors in log-scale, involving the (relative) edge of stability $\chi$ ({\color{red!50} red} area and line), the spatial participation factor $\varsigma$ ({\color{blue!50} blue} area and line), and the temporal calibration factor $\tau$ ({\color{white!30!black} grey} overlay and black line).
Note that $\log \tau$ operates in the opposite direction by subtraction.
}
\label{fig:decomposition-2}
\vspace*{\fill}
\end{figure}

\begin{figure}[p]
\centering
\vspace*{\fill}
\includegraphics[width=\linewidth]{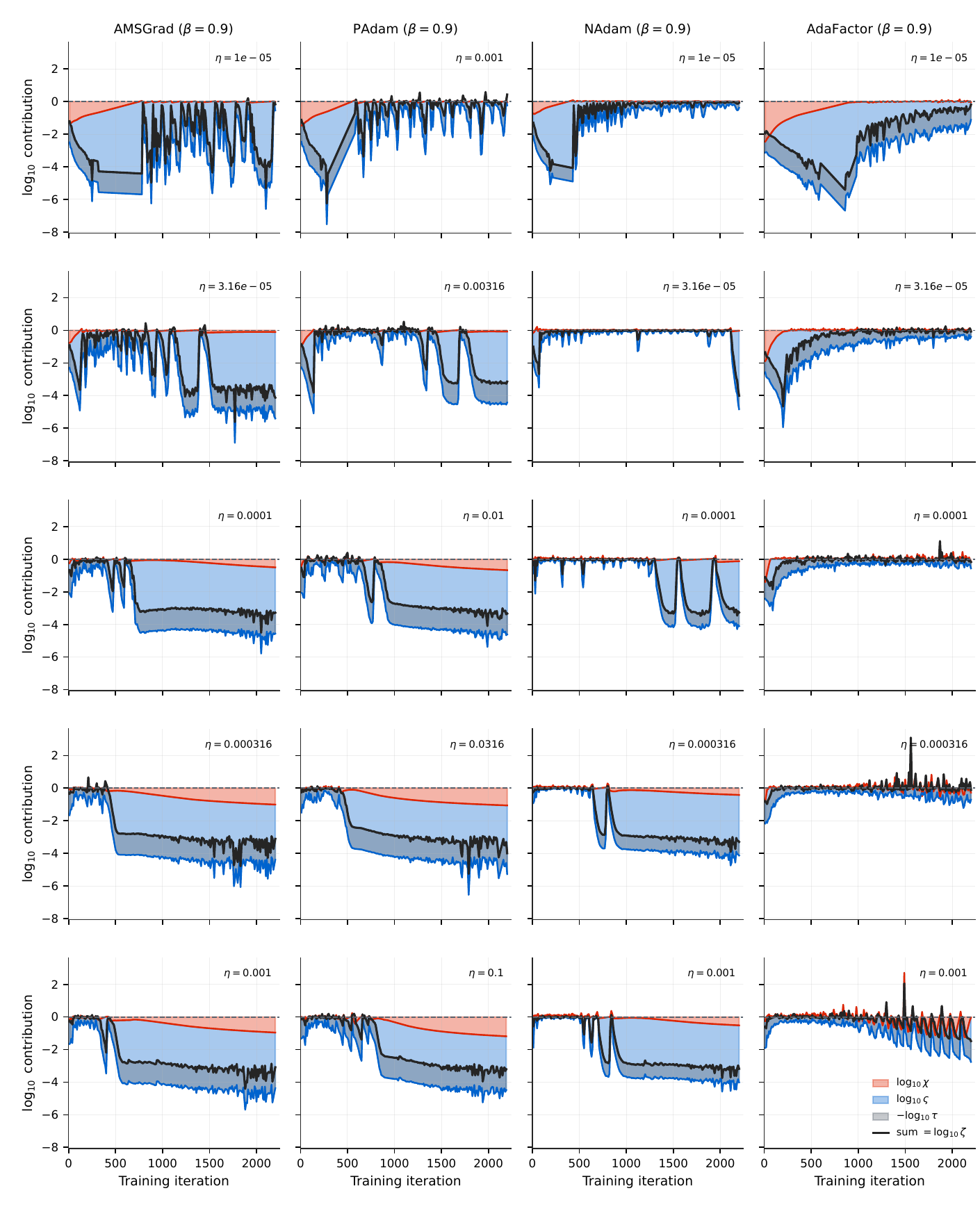}
\vspace{-1em}
\caption{\small
Decomposition of the realized EoS $\zeta$ into factors in log-scale, involving the (relative) edge of stability $\chi$ ({\color{red!50} red} area and line), the spatial participation factor $\varsigma$ ({\color{blue!50} blue} area and line), and the temporal calibration factor $\tau$ ({\color{white!30!black} grey} overlay and black line).
Note that $\log \tau$ operates in the opposite direction by subtraction.
}
\label{fig:decomposition-3}
\vspace*{\fill}
\end{figure}

\begin{figure}[p]
\centering
\vspace*{\fill}
\includegraphics[width=\linewidth]{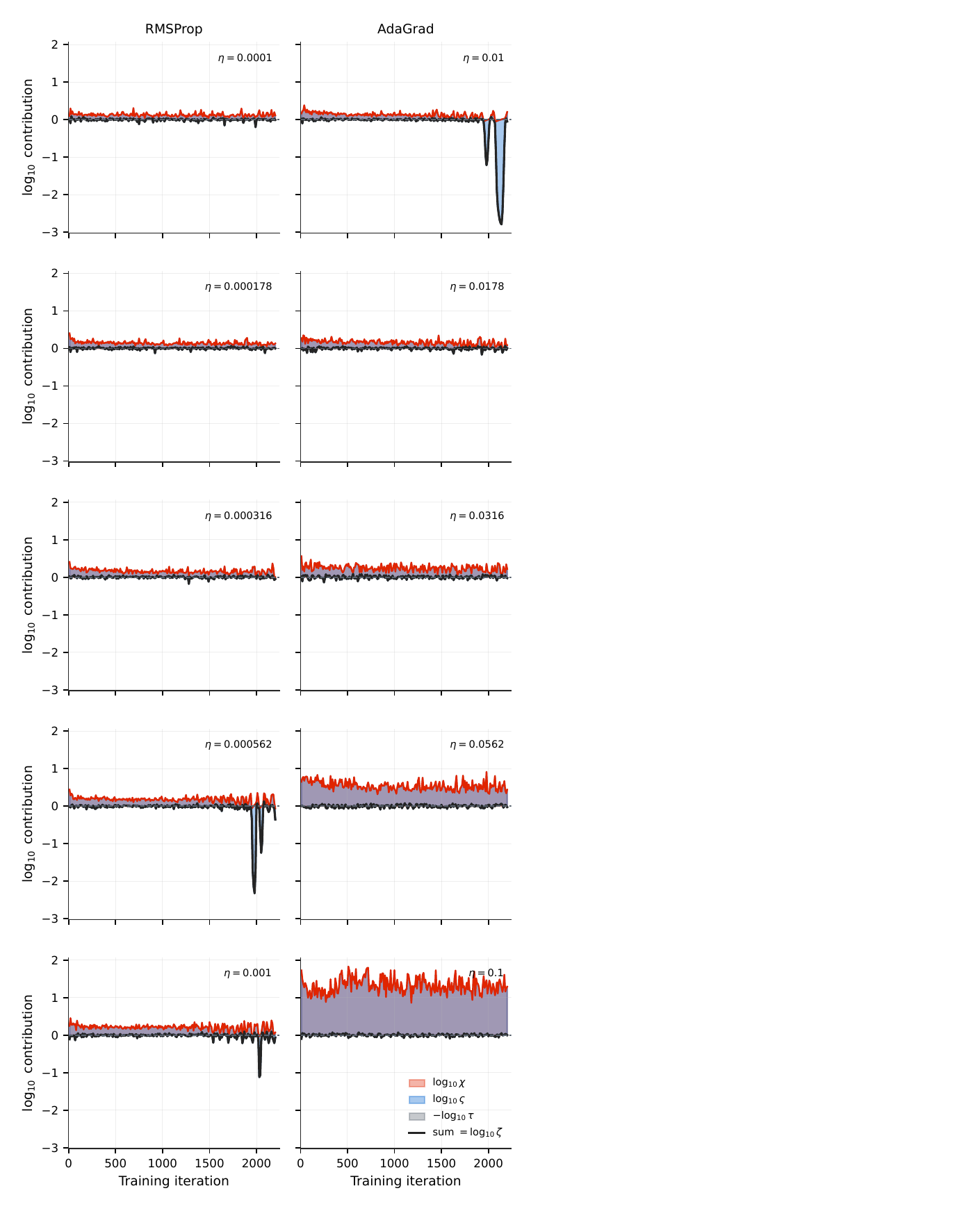}
\vspace{-1em}
\caption{\small
Decomposition of the realized EoS $\zeta$ into factors in log-scale, involving the (relative) edge of stability $\chi$ ({\color{red!50} red} area and line), the spatial participation factor $\varsigma$ ({\color{blue!50} blue} area and line), and the temporal calibration factor $\tau$ ({\color{white!30!black} grey} overlay and black line).
Note that $\log \tau$ operates in the opposite direction by subtraction.
}
\label{fig:decomposition-4}
\vspace*{\fill}
\end{figure}

\end{document}